\documentclass{article}
\usepackage{microtype}
\usepackage{graphicx}
\usepackage{booktabs}
\usepackage{amsmath,amssymb,amsthm}
\usepackage{xcolor}
\usepackage{enumitem}
\usepackage{hyperref}

\usepackage[accepted]{mlsys2025}
\makeatletter
\renewcommand{\mlsys@appearing}{Preprint.}
\makeatother
\usepackage{tikz}
\usepackage{longtable}
\usepackage[section]{placeins}
\usetikzlibrary{positioning,arrows.meta}
\usepackage[capitalise,noabbrev]{cleveref}

\newcommand{\OPT}{\mathrm{OPT}}
\newcommand{\poly}{\mathrm{poly}}
\newcommand{\RNCases}{78}
\newcommand{\RNArch}{16}

\newcommand{\RtransformerN}{15}
\newcommand{\RtransformerPos}{14}

\newcommand{\RtransformerMax}{2.75}
\newcommand{\RtransformerMedian}{1.30}
\newcommand{\RcnnN}{8}
\newcommand{\RcnnPos}{3}

\newcommand{\RcnnMax}{3.02}
\newcommand{\RcnnMedian}{0.00}
\newcommand{\RsweepN}{18}
\newcommand{\RsweepPos}{17}

\newcommand{\RsweepMax}{3.46}
\newcommand{\RsweepMedian}{1.30}
\newcommand{\RreplicationN}{9}
\newcommand{\RreplicationPos}{7}

\newcommand{\RreplicationMax}{7.60}
\newcommand{\RreplicationMedian}{2.54}
\newcommand{\RRegN}{54}
\newcommand{\RRegMispred}{22}

\newcommand{\RRegMeanRegen}{0.94}
\newcommand{\RRegMeanFixed}{1.97}
\newcommand{\RRegWorse}{9}

\newcommand{\RRegWorstAbs}{30.6}

\newcommand{\RCertGapMax}{1.50}
\newcommand{\RCertGapMin}{0.36}

\newcommand{\RCertUnitsMax}{259}
\newcommand{\RCpN}{94}
\newcommand{\RCpProved}{91}
\newcommand{\RCpUnknown}{3}
\newcommand{\RCpMedian}{0.02}
\newcommand{\RCpMax}{66}
\newcommand{\RCertMedianCp}{0.004}
\newcommand{\RCertMaxCp}{0.25}

\newcommand{\RLatMin}{0.959}
\newcommand{\RLatMax}{1.045}
\newcommand{\RLatMedian}{0.997}
\newcommand{\RTOneN}{800}
\newcommand{\RTOneFR}{773}
\newcommand{\RTOneG}{671}

\newcommand{\RTOneWorst}{1.23}
\newcommand{\RTOneRgtF}{98}
\newcommand{\RTOneFgtR}{11}
\newcommand{\RTTwoN}{375}
\newcommand{\RTTwoHelps}{117}
\newcommand{\RTTwoRatio}{0.942}
\newcommand{\RTTwoB}{290}
\newcommand{\RTTwoCert}{270}
\newcommand{\RTTwoCover}{43}
\newcommand{\RAdvWorst}{1.50}
\newcommand{\RGpuN}{50}

\newcommand{\RGpuDetBit}{27}
\newcommand{\RGpuNondet}{23}
\newcommand{\RGpuExact}{17}
\newcommand{\RGpuMod}{49}
\newcommand{\RGpuPos}{25}
\newcommand{\RGpuNeg}{8}
\newcommand{\RGpuMax}{7.4}
\newcommand{\RGpuMedPos}{1.3}
\newcommand{\RGpuWorst}{1.5}

\newcommand{\RGpuRegenWorse}{19}
\newcommand{\RGpuTimeMed}{0.998}
\newcommand{\RGpuGuardMean}{0.92}
\newcommand{\RSimN}{24}
\newcommand{\RSimDown}{6}
\newcommand{\RSimEq}{13}
\newcommand{\RSimUp}{5}
\newcommand{\RInfN}{22}

\newcommand{\RInfDeeplab}{10.0}
\newcommand{\RInfExact}{22}
\newcommand{\RInfGraphs}{22}

\newcommand{\RInfStockOpt}{20}
\newcommand{\RCertAllN}{122}
\newcommand{\RCertAll}{115}
\newcommand{\RCertAllPct}{94}
\newcommand{\RCertAllF}{99}
\newcommand{\RCertAllR}{112}
\newcommand{\RInfGap}{13}
\newcommand{\REdgeN}{16}

\newcommand{\REdgeMaxN}{134}
\newcommand{\RBfN}{6}
\newcommand{\RBfExact}{6}
\newcommand{\RPreN}{6}
\newcommand{\RPreExact}{6}
\newcommand{\RPreSame}{6}

\newcommand{\RPubMin}{5.2}
\newcommand{\RPubMax}{6.1}
\newcommand{\RPubExtraMax}{0.34}
\newcommand{\RPubExtraMin}{0.02}

\newtheorem{theorem}{Theorem}
\newtheorem{lemma}[theorem]{Lemma}
\newtheorem{proposition}[theorem]{Proposition}

\theoremstyle{definition}
\newtheorem{definition}[theorem]{Definition}
\theoremstyle{remark}

\setlist{nosep,leftmargin=*}

\mlsystitlerunning{Peak-Memory Schedules You Can Trust}

\begin{document}

\twocolumn[
\mlsystitle{Peak-Memory Schedules You Can Trust: Exact Realization, Certified Optimality, and the Limits of Recomputation for Inference}

\begin{mlsysauthorlist}
\mlsysauthor{Ivan Fomin}{yandex}
\mlsysauthor{Ivan Koryabkin}{yandex}
\mlsysauthor{Boris Sheludko}{yandex}
\end{mlsysauthorlist}
\mlsysaffiliation{yandex}{Yandex, Moscow, Russia}
\mlsyscorrespondingauthor{Ivan Fomin}{ash-ivanfomin@yandex-team.ru}
\mlsyscorrespondingauthor{Ivan Koryabkin}{koryabkin@yandex-team.ru}
\mlsyscorrespondingauthor{Boris Sheludko}{bsheludko@yandex-team.ru}
\mlsyskeywords{Memory scheduling, Compilers, Rematerialization}

\vskip 0.3in
\begin{abstract}
Real-time inference on small devices needs low activation memory without extra latency, which makes operator order the most valuable memory lever: it is free in time.
We show that the usual way of deploying a better order, handing it back to the compiler and regenerating code, is unreliable.
In PyTorch Inductor the regenerated program misses its predicted peak in \RRegMispred{} of \RRegN{} CPU configurations.
It is worse than the stock program in \RRegWorse{} of them and in \RGpuRegenWorse{} of \RGpuN{} on a GPU, and on the GPU it also changes numerical results.
We instead profile the generated program, build an exact model of its memory behaviour from its own statements and their read/write effects, and permute those statements.
We prove that this realizes the modeled peak whenever allocation does not depend on history.
It does so to the byte, with bitwise-identical results, in every one of our CPU runs.
On the GPU we characterize why the caching allocator breaks the condition.
To know when to stop, we give forward and reverse ``trap'' lower bounds, each computable with $O(n^2)$ minimum cuts.
They certify the found order optimal for \RCertAllPct\% of held-out wrapper graphs and for all \REdgeN{} microcontroller-class models we study, in under a second.
The certificates show that most inference graphs are already scheduled optimally.
Branching networks are the exception, where reordering is free and pays, e.g.\ \RInfDeeplab\% on DeepLabV3.
In compiled training we save up to 7.6\%.
Beyond the certified optimum we study selective recomputation with per-operator caps.
It is strongly NP-hard for every fixed cap even at treewidth two, it encodes knapsack exactly on series--parallel graphs, and it is solvable exactly on bounded treewidth.
It is useless on public tiny models whose peak is a single layer, but on architectures that hold a cheap value for a late reuse it pays: on the public ReDimNet speaker networks, whose order is already optimal, it lowers the peak by \RPubMin--\RPubMax\% at under 0.4\% extra MACs, with bitwise-identical outputs.
\end{abstract}
]
\printAffiliationsAndNotice{}


\section{Introduction}
\label{sec:intro}

A neural network that runs in real time on a phone, a hearing aid or a microcontroller has to fit its activations into a fixed, small memory budget, and it has to do so without slowing down.
Most of the familiar levers for reducing activation memory trade it for time.
Rematerialization recomputes values instead of storing them~\cite{chen2016training,jain2020checkmate}.
Offloading moves them to slower memory.
Patch-based execution recomputes overlapping receptive fields~\cite{lin2021mcunetv2}.
For training these trades are usually acceptable, because a few percent of throughput buys a larger batch.
For real-time inference they often are not: a keyword spotter or a streaming speech model has a hard deadline per frame.
The interesting levers there are the ones that are free in time, such as the order in which operators run, and the recomputations that cost very little time but remove a large buffer.

Operator order is already exploited.
PyTorch Inductor ships a peak-memory reordering pass that is enabled by default~\cite{ansel2024pytorch2}.
TensorFlow Lite Micro plans a static arena for the order it is given~\cite{david2021tflm}.
A line of work schedules irregular graphs exactly or hierarchically~\cite{ahn2020serenity,wang2022hmcos,steiner2023model,jin2023newtools}.
These systems share an implicit assumption: if a better order is found on the compiler's graph, the deployed program will be better by the predicted amount.
They also leave a practical question unanswered, namely when to stop.
A deployment engineer who has run a scheduler cannot tell whether the result is optimal or whether a better order exists that the heuristic missed.

This paper makes peak-memory scheduling trustworthy in both senses.
We show that the assumption above is false in a mainstream compiler, give a realization method for which it is true by construction, and give certificates that settle the question of when to stop.
We then ask what selective recomputation can add once the order is exhausted.
The answer is a mix of good and bad news that we think is useful to state precisely.

\paragraph{Regeneration is unreliable.}
The usual way to deploy a better order is to hand it back to the compiler and regenerate code.
Code generation, however, runs \emph{after} scheduling.
It decides how kernels are grouped into calls, which dead buffers are reused in place, and when references are dropped, and a new order changes those decisions.
Across \RRegN{} compiled CPU configurations of \RNArch{} architectures, the regenerated program misses its own predicted peak in \RRegMispred{} cases and is worse than the stock program in \RRegWorse.
On a GPU it is worse than stock in \RGpuRegenWorse{} of \RGpuN{} configurations and, in addition, changes numerical results (\cref{sec:regen}).

\paragraph{Fixed-wrapper realization is exact.}
We instead take the generated program as it is, profile one execution, and build an exact model of its memory behaviour at the granularity of its statements.
We then permute the program's own statements.
Kernels, arguments and buffer reuse stay untouched; only the order and the placement of allocations and releases change.
Under an explicit contract the realized peak equals the modeled peak (\cref{thm:lift}).
On CPU this holds to the byte in all \RNCases{} training runs and all \RInfN{} inference runs, and outputs and gradients are bitwise identical to the stock program.
On a GPU the CUDA caching allocator violates one condition of the contract, and we characterize exactly how (\cref{sec:gpu}).

\paragraph{Certificates tell us when to stop.}
We give two lower bounds on the optimal peak, a forward and a reverse ``trap'' bound.
Each is computable with $O(n^2)$ minimum cuts, and each dominates the classical event-cut bound.
Together they prove the found order optimal for \RCertAll{} of \RCertAllN{} held-out wrapper graphs, in well under a second, and for all \REdgeN{} microcontroller-class models we study.
This changes the practical conversation.
For most inference graphs the certificate shows that the stock order is \emph{already optimal}, so no scheduler can help.
The exception is branching networks, where reordering is free and the saving can be large: \RInfDeeplab\% on DeepLabV3 inference.

\paragraph{The limits of recomputation.}
When the certified optimum is still above the budget, recomputation is the next lever, and for inference it is attractive only if it is selective, recomputing a few cheap values rather than whole subgraphs.
We model this with per-operator recomputation caps $k_v$, as in Moccasin~\cite{bartan2023moccasin}, and prove three things.
\begin{itemize}
\item Peak feasibility is strongly NP-hard for every fixed uniform cap, already at treewidth two.
\item Minimizing extra work under cap two is exactly a knapsack problem on series--parallel graphs.
\item An exact algorithm exists on bounded treewidth, but with a necessary dependence on the numeric budget.
\end{itemize}
Empirically, recomputation at tensor granularity cannot help the public tiny models we study.
In every one of them the peak is set by a single operator's inputs and output, a bound no schedule of whole tensors can beat (\cref{sec:tiny}).
Recomputation pays off where a large value is held across the network only for a late reuse and can be recreated cheaply from a small ancestor.
The public ReDimNet speaker-embedding networks have exactly this structure.
In all seven sizes their order is already optimal, yet recomputing the stem once lowers the peak by \RPubMin--\RPubMax\% at no more than \RPubExtraMax\% extra MACs (\cref{sec:rdn}).
For such values our model of caps and our certificates apply directly; for tiny models whose peak is a single layer, savings require splitting tensors, as patch-based inference does~\cite{lin2021mcunetv2}.

We do not propose a new search heuristic; the search in our pipeline is an existing exact capacity search.
Our contributions are the realization method, the certificates, the measurements that show where scheduling can and cannot help, and the complexity results that delimit selective recomputation.

\section{Background and Model}
\label{sec:model}

\paragraph{The one-shot root model.}
A schedule executes each operation of a directed acyclic graph once.
Operation $v$ allocates new buffers (``roots'') of total size $a(v)$ and may use transient memory, and both are allocated before any of its inputs can be released.
A root is live after a set $S$ of executed operations if it has been produced and is either pinned (it outlives the program, like a returned output) or still has an unexecuted consumer.
Writing $L(S)$ for the live bytes and $s(v)\ge0$ for the transient envelope of $v$, executing $v$ after $S$ costs $L(S)+a(v)+s(v)$.
The peak of an order is its largest event cost, and $\OPT$ is the smallest peak over topological orders.
Multi-output calls, views and external inputs are expressed by roots with several consumers.
$L$ is submodular on the lattice of ideals, which is what makes minimum-cut arguments work below.

\paragraph{Compiled programs.}
PyTorch Inductor lowers a model to a flat Python \emph{wrapper} that allocates buffers, calls generated C++ or Triton kernels and library kernels (matrix multiplication, convolution, attention), aliases dead buffers for reuse, and drops references with \texttt{del}.
For training, AOTAutograd produces one wrapper for the forward and one for the backward pass.
The point that matters for this paper is that storage is born and freed by the wrapper, not by the graph the scheduler sees.
Deployment runtimes for microcontrollers, such as TFLite Micro, instead execute operators in model order and place every tensor in a statically planned arena~\cite{david2021tflm}.
On GPUs, PyTorch uses a caching allocator that rounds requests and reuses cached blocks.
\Cref{sec:gpu} shows that this difference between static, exact allocation and history-dependent caching is decisive for whether a modeled peak is realized.

\section{Realizing a Schedule Exactly}
\label{sec:realize}

\begin{figure*}[t]
\centering
\begin{tikzpicture}[node distance=4mm and 7mm, font=\footnotesize,
  box/.style={draw=black!55, rounded corners=2pt, align=center, minimum height=9mm, text width=24mm, fill=white},
  acc/.style={box, draw=blue!60!black, fill=blue!6},
  bad/.style={box, draw=orange!70!black, fill=orange!7},
  arr/.style={-{Stealth[length=2mm]}, draw=black!60}]
\node[box] (m) {model + inputs};
\node[box, right=of m] (c) {stock compile\\(Inductor wrapper)};
\node[box, right=of c] (p) {profile one step\\per statement};
\node[box, right=of p] (r) {root model\\units, R/W effects};
\node[box, right=of r] (s) {exact search\\+ certificates};
\node[acc, below=7mm of s] (f) {\textbf{fixed wrapper}\\permute statements};
\node[bad, below=7mm of c] (g) {regenerate code\\with forced order};
\node[box, below=7mm of f, xshift=-17mm, text width=40mm] (v) {re-profile: measured peak $=$ model?\\(guard: keep stock if worse)};
\draw[arr] (m) -- (c); \draw[arr] (c) -- (p); \draw[arr] (p) -- (r); \draw[arr] (r) -- (s);
\draw[arr] (s) -- (f); \draw[arr] (s.south west) to[bend left=12] node[above, font=\scriptsize, sloped]{usual pipeline} (g.north east);
\draw[arr] (f) -- (v); \draw[arr] (g.south) |- node[pos=0.75, below, font=\scriptsize]{changes kernels, reuse, numerics} (v.west);
\end{tikzpicture}
\caption{Fixed-wrapper realization (blue) versus the usual regeneration pipeline (orange). The order is optimized on a model built from the generated program itself, and then applied to that program by permuting its statements, so the kernels, their arguments and buffer reuse are those that were profiled.}
\label{fig:pipeline}
\end{figure*}
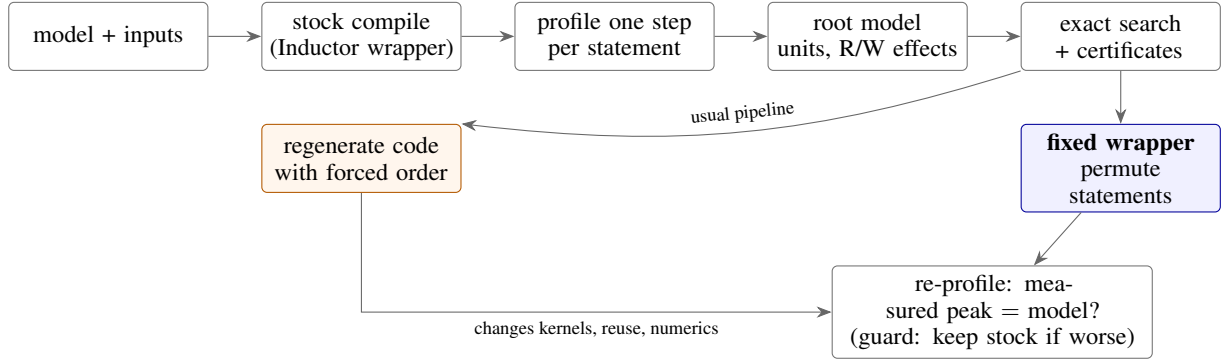

\Cref{fig:pipeline} summarizes the method.
We compile the model once with the stock compiler and capture every generated wrapper.
We then run one warm step in which each wrapper statement executes inside a profiler scope.
After each statement we record the storage address and size bound to every name it assigns.
Every allocation and free event is thereby attributed to a statement.
Storages are identified by allocation epoch rather than by address, because the allocator reuses addresses.

\paragraph{From statements to a root model.}
Each statement that calls a kernel or a library function becomes a \emph{unit}.
A statement that only allocates a buffer joins the first unit that touches that buffer, so allocation happens as late as possible.
Aliases (views, tuple projections and the reuse assignments Inductor emits) and \texttt{del} statements are not units; they are placed during emission.
A storage that is allocated and freed inside a single statement is a transient and contributes to that unit's envelope $s(v)$.
Every other storage born in the program becomes a root produced by its unit.
Its consumers are the units that touch any name bound to it, and it is pinned if it outlives the call.
Because a profiled storage is identified with all the names that ever referred to it, ownership through views and tuples is accounted for exactly.

\paragraph{Dependencies from read/write effects.}
Ordering constraints come from data flow at storage level.
For generated C++ kernels we read the const-ness of every pointer parameter from the emitted signature.
For Triton kernels we use the role encoded in the parameter name (\texttt{in\_ptr}, \texttt{out\_ptr}, \texttt{in\_out\_ptr}).
For library calls, the \texttt{out=} argument and the first argument of a mutating operator are writes.
Calls we do not recognize are serialized with every other unit.
Per storage we then add read-after-write, write-after-read and write-after-write edges, which leaves readers of the same buffer generation free to reorder.
This distinction turns out to matter more than the search itself (\cref{sec:train}).

\paragraph{Emission.}
Given a unit order, we emit the original statements, each exactly once.
Allocations go immediately before their unit, and aliases, assertions and \texttt{del}s go as early as their dependencies allow.
No statement is synthesized and no release is invented.
The one exception is renaming: a few generated wrappers rebind an input name, and we rename it so that every name is assigned once, deleting the old binding where the original rebinding would have dropped it.
The reordered function is compiled into the same module, so it calls the very kernel objects that were profiled.

\begin{theorem}[Exact realization]\label{thm:lift}
Assume the generated program satisfies a contract with six conditions:
(C1)~region-level effects are complete;
(C2)~ownership and completion are as recorded;
(C3)~each unit computes a deterministic function of its inputs;
(C4)~each unit's sequence of allocation events is as recorded;
(C5)~that sequence is independent of the surrounding live memory (the frame condition);
(C6)~emission introduces no other events.
Then every unit order that respects the dependencies preserves the program's results, bitwise if the kernels are deterministic, and realizes exactly the peak predicted by the root model.
\end{theorem}

The proof (\cref{app:lift}) is an event-by-event simulation that maintains a correspondence between modeled and concrete live storage, owners and buffer generations.
Its ingredients are standard.
Staged effects exist in MLIR~\cite{mlir}, and heap simulation over an instruction sequence exists in XLA~\cite{xla-heap}.
What is specific is that the composition survives \emph{direct} permutation of a generated program that calls opaque library kernels.
The contract is not proven for library kernels, so our harness never trusts it.
Every emitted program is re-profiled, and the measured peak of every wrapper is compared with the prediction.
During development this check caught two bugs of our own.
An alias used only by the final return was emitted too late, and it kept a tuple, and with it a 16\,KB gradient, alive.
Stochastic-depth layers in some vision models made even the stock program non-reproducible.
Both were fixed before the runs reported here.

\paragraph{Search.}
The root models have up to \RCertUnitsMax{} units.
We search with an exact capacity search: depth-first search over ideals with forced non-increasing steps and minimum-cut packet closure.
It bisects the capacity between the lower bound of \cref{sec:cert} and the peak of the stock order.
In almost all graphs it reaches the certified bound within milliseconds, so the search is not where the difficulty lies.

\section{Certifying That No Better Order Exists}
\label{sec:cert}

A scheduler that returns a peak $U$ is only half an answer; the other half is a lower bound $\mathrm{LB}$ with $\mathrm{LB}=U$, which proves that no order is better.
For an operation $v$ let $A(v)$ be its strict ancestors and $D(v)$ its descendants including $v$.
The classical \textbf{event-cut bound} asks, for every $v$, how small the live memory can possibly be just before $v$ runs:
\[
G=\max_{v}\Big[a(v)+s(v)+\min_{X\text{ ideal}:\,A(v)\subseteq X,\ v\notin X} L(X)\Big].
\]
Each inner minimum is a minimum-weight closure, i.e.\ one minimum cut.
The bound can be loose because the cheapest prefix for one operation may be incompatible with the cheapest prefix for another.

\begin{definition}[Trap bounds]
Let $c(v)=a(v)+s(v)$.
For a proper ideal $R$ with ready frontier $F(R)$, let $f_v(R)=c(v)+\min\{L(X): A(v)\subseteq X\subseteq R\}$ and $\mathcal F=\max_R \min_{v\in F(R)} f_v(R)$.
For a nonempty ideal $T$ and maximal $v\in T$, let $q_v(T)=c(v)+\min\{L(X): T\setminus\{v\}\subseteq X\subseteq V\setminus D(v)\}$ and $\mathcal R=\max_T\min_{v} q_v(T)$.
\end{definition}

The forward bound $\mathcal F$ rests on a simple observation.
Every order leaves any proper ideal $R$ through some first operation, and at that moment its actual prefix lies between that operation's ancestors and $R$.
The reverse bound $\mathcal R$ uses the last operation of any nonempty ideal $T$ instead.
Neither argument assumes that the minimizing prefix is reachable, which is why both are sound but can be loose.

\begin{theorem}\label{thm:traps}
(a)~$G\le\mathcal F\le\OPT$ and $G\le\mathcal R\le\OPT$.
(b)~Each of $\mathcal F$ and $\mathcal R$ can be computed with $O(n^2)$ minimum-cut calls, independently of the magnitude of the weights.
(c)~The two bounds are incomparable.
(d)~$\OPT/\mathcal F$ is unbounded even on binary in-trees and on a family of pairwise merges, while $\mathcal R=\OPT$ on both families.
\end{theorem}

Domination holds because the ideal $V\setminus D(v)$ has $v$ as its only ready operation, which reproduces the event-cut term for $v$.
Computation uses monotone optimistic saturation.
For $\mathcal F$, we repeatedly admit a frontier operation of smallest $f_v(R)$ and record the running maximum; the reverse bound removes maximal operations symmetrically.
Proofs, both strict separations and the two families appear in \cref{app:traps}.
A certificate consists of the stuck ideal and one flow/cut pair per frontier operation.
Our verifier recomputes liveness from the graph, checks every cut value against its flow and replays the upper order, without importing the solver or running max-flow.

Whether $\OPT/\max(\mathcal F,\mathcal R)$ is bounded is open.
\Cref{fig:certificates}a shows how tight the bounds are in practice.
On \RTOneN{} random DAGs of up to 14 vertices (series--parallel, pumpkin, tree, random and training-shaped), compared against the exact optimum, $\max(\mathcal F,\mathcal R)$ is exact in \RTOneFR{} cases.
The event-cut bound $G$ is exact in only \RTOneG.
The reverse bound is the stronger of the two: it beats $\mathcal F$ in \RTOneRgtF{} instances and loses in \RTOneFgtR.
The worst ratio we found is \RTOneWorst{} on random graphs and \RAdvWorst{} with an adversarial hill-climb over 13-vertex graphs.

\begin{figure*}[t]
\centering
\includegraphics[width=\textwidth]{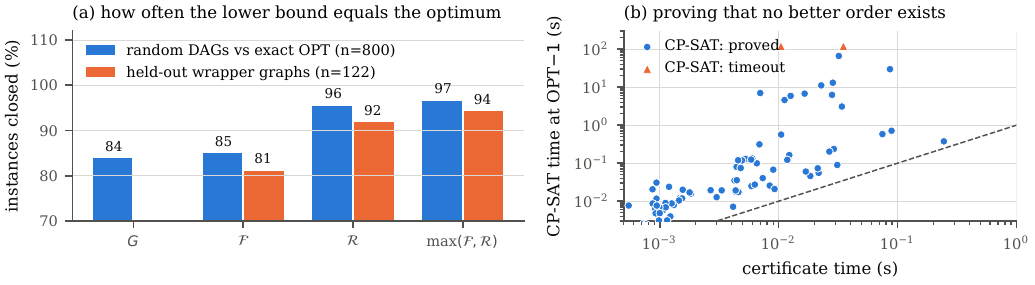}
\caption{(a) Fraction of instances on which a lower bound equals the optimum. Blue: random DAGs with exactly computed optima. Orange: forward and backward wrapper graphs of the held-out training and inference models, where the reference is the peak of the found order. (b) Time to certify that $\OPT-1$ is infeasible, trap certificates versus an exact CP-SAT cumulative model with one worker and a 120\,s limit, on the certified graphs.}
\label{fig:certificates}
\end{figure*}

\section{Evaluation}
\label{sec:eval}

\paragraph{Setup.}
All runs use PyTorch 2.14.1 with Inductor.
On CPU, every run uses FP32 and one thread; we also repeat the inference grid in bfloat16.
The GPU is an NVIDIA RTX 2080 Ti with a CUDA~12.6 build of the same release.
The stock baseline is Inductor with \texttt{reorder\_for\_peak\_memory} and in-place buffer reuse enabled; only the compile cache is disabled.
Dropout and stochastic depth are disabled.
Weights are random unless stated otherwise; \cref{sec:precision} shows that pretrained weights give identical schedules and peaks.
Inference runs evaluate the model in evaluation mode without gradients.
Training runs measure one compiled forward pass, an MSE loss against a fixed random target and the backward pass, without an optimizer step.
Peak is the maximum of tracked tensor allocations inside the measured step; it excludes the resident weights and inputs, which we add back where we report savings ``including the frame''.
Every run measures the stock, the normalized (our emission of the stock unit order) and the reordered program twice, and compares outputs and all gradients bitwise across them.

The study has four parts:
\begin{itemize}
\item server-scale inference: \RInfN{} configurations, including segmentation networks, CNNs, and Transformer prefill with long sequences;
\item microcontroller-class inference: \REdgeN{} models, among them the four MLPerf Tiny reference networks~\cite{banbury2021mlperftiny,zhang2017helloedge}, TC-ResNet keyword spotters~\cite{choi2019tcresnet}, small MobileNets, MNASNet, ShuffleNet, SqueezeNet, EfficientNet-B0, GoogLeNet and a tiny U-Net;
\item compiled training: 50 configurations of 16 architectures, plus shape sweeps, ablations and seed repeats;
\item the GPU replication.
\end{itemize}

\subsection{Regeneration is unreliable}
\label{sec:regen}

For every configuration we also map the optimized unit order back to Inductor's scheduler nodes, force it in the reordering pass, recompile, and measure.
\Cref{fig:regen} plots the outcome against the saving that fixed-wrapper realization achieves with the same order.
If regeneration were faithful, all points would lie on the diagonal.
On CPU the regenerated program misses its predicted peak in \RRegMispred{} of \RRegN{} training configurations.
Where a saving exists it realizes \RRegMeanRegen\% on average instead of \RRegMeanFixed\%.
In \RRegWorse{} configurations it is worse than stock, once by \RRegWorstAbs\% (T5 encoder at sequence length 512).
On the GPU the picture is worse: regeneration is worse than stock in \RGpuRegenWorse{} of \RGpuN{} configurations.
The post-fusion graph is unchanged in all of these cases.
What changes is how kernels are grouped into calls and which dead buffers are reused; for the encoder--decoder, 559 backward statements become 554.
Since recompiled kernels also round differently, regeneration on GPU additionally changes gradients, while fixed-wrapper realization does not.
A memory scheduler that may silently make a program worse, and change its numerics, is hard to deploy; this is the main reason we realize schedules on the generated program.

\begin{figure}[t]
\centering
\includegraphics[width=\columnwidth]{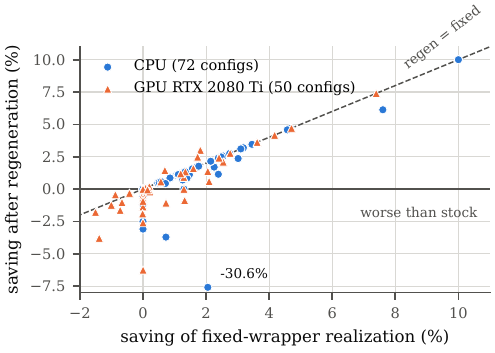}
\caption{The same optimized order, realized by permuting the generated program (x) versus regenerating code for it (y). Points below the diagonal lose part of the saving and points below zero are worse than the stock program; one CPU point at $-30.6\%$ is clipped.}
\label{fig:regen}
\end{figure}

\subsection{Inference: where reordering helps and where it cannot}
\label{sec:inference}

For server-scale inference, fixed-wrapper realization is exact in all \RInfExact{} configurations, with zero prediction error and bitwise-identical outputs.
The saving, however, is concentrated.
Only one configuration improves: DeepLabV3-ResNet50 at $512\times512$, whose peak drops by \RInfDeeplab\%.
DeepLabV3's atrous spatial pyramid has five parallel branches over a large feature map, and the stock order keeps several of their outputs alive at once.
In the other configurations the certificates prove the stock order optimal for \RInfStockOpt{} of the \RInfGraphs{} inference graphs; the one open case is discussed in \cref{sec:certeval}.
This includes ResNet, MobileNetV2, EfficientNet, DenseNet, Inception, Swin, ViT, BERT, GPT-2 and Llama prefill at up to 2048 tokens.
Inference graphs are mostly chains of blocks whose activations die immediately, so a single sensible order is all there is.
The practical value of the certificate is precisely to say so, and to direct effort elsewhere.

\label{sec:tiny}
\paragraph{Public microcontroller-class models.}
For the tiny models we build the operator graph the way an embedded runtime sees it.
Batch normalization is folded into convolutions, activations are fused into their producer, views are aliases, and activations are int8.
We compare the order in which TFLite Micro would execute the model with the certified optimum, and evaluate the arena with a TFLite-Micro-style greedy planner.
The certificates prove the model order already optimal for all \REdgeN{} models (up to \REdgeMaxN{} operators), and executing every schedule reproduces the modeled peak exactly.
More importantly, in every one of them the peak equals a trivial lower bound: the inputs plus the output of a single operator.
\Cref{fig:tiny} shows why.
The memory profile of a tiny CNN has a sharp early spike at a high-resolution layer (MobileNets), or one at a skip-concatenation (U-Net), and the rest of the network is far below it.
No schedule of whole tensors, with or without recomputation, can go below that single-operator bound.
The capped-recomputation solver, given a cap of two on every operator, accordingly never finds a feasible reduction.
These models are deliberately designed so that the early high-resolution layer dominates, and further savings for them require splitting that layer's tensors into tiles, as patch-based inference does~\cite{lin2021mcunetv2,alwani2016fused}.
Tiling recomputes the halo regions between tiles, which is exactly the selective recomputation with per-vertex caps studied in \cref{sec:remat}, now on a tile-level graph.

\begin{figure}[t]
\centering
\includegraphics[width=\columnwidth]{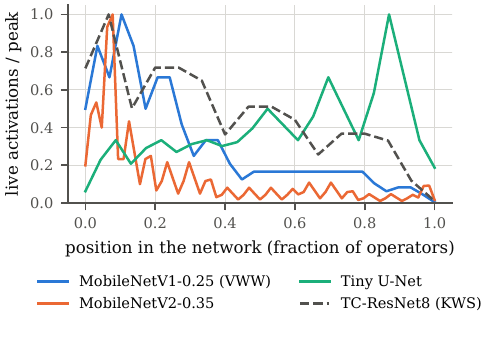}
\caption{Live activation memory along four microcontroller-class models, executed in model order, normalized by the peak. Each peak is set by a single operator (an early high-resolution layer, or a U-Net skip concatenation), which bounds what any schedule of whole tensors can achieve.}
\label{fig:tiny}
\end{figure}

\subsection{Selective recomputation on a speaker-embedding network}
\label{sec:rdn}

ReDimNet~\cite{yakovlev2024redimnet} is a family of speaker-embedding networks designed for on-device speaker recognition, in seven sizes (B0--B6).
We take the public VoxCeleb2 checkpoints and trace each backbone at the operator level, fusing normalization and activations into their producers, with fp32 activations.
Inputs are 2\,s and 0.8\,s of audio (201 and 81 feature frames); the latter is the typical window of a streaming engine.
At every stage ReDimNet aggregates the outputs of all previous stages, and every stage also reads the stem output again.
The stem output is therefore held across the whole network for these late reuses, although it is a short and cheap function of the input features, which are small.

For every model the certificates prove the model order already optimal, so reordering alone gains nothing.
We then keep the small input features alive and recompute the stem chain once from them for its late uses.
The transformed graph is rescheduled exactly, and the certificates prove the new order optimal for it as well.
Executing the resulting schedule reproduces the modeled peak, and the output is bitwise identical to the original model.
\Cref{tab:rdnpub} shows the 2\,s results; the 0.8\,s results have the same relative savings.
The peak falls by \RPubMin--\RPubMax\% at \RPubExtraMin--\RPubExtraMax\% extra MACs, the smaller models paying more because their stem is a larger share of the work.
For B3 and B6 a greedy arena planner loses the saving to fragmentation; a joint order-and-placement step would be needed there.
Recomputing the stage outputs themselves would not help, because each of them is a direct input of every later aggregation, including the one at the peak.
The pattern that rewards selective recomputation is a value held only for a \emph{late} reuse that does not coincide with the peak event.

On internal production on-device models, which we cannot release, the same selective recomputation reduced the deployed TFLite activation arena by up to 15\% at about 1\% extra MACs, with bitwise-identical outputs.
These numbers are reported for context only; all claims in this paper rest on the public models above.

\begin{table}[t]
\centering\scriptsize
\caption{Selective recomputation on the public ReDimNet backbones (2\,s input, fp32 activations). Peak live activation memory (MB) of the certified optimal order, before and after recomputing the stem for its late uses, and the extra MACs. All outputs are bitwise identical; both orders are certified optimal for their graphs.}
\label{tab:rdnpub}
\setlength{\tabcolsep}{2.5pt}
\begin{tabular}{@{}lrrrrr@{}}
\toprule
Size & Ops & GMACs & Opt.\ order & Recompute & MACs \\
\midrule
B0 & 309 & 0.59 & 7.28 & 6.85 ($-$6.0\%) & $+$0.33\% \\
B1 & 223 & 0.85 & 10.49 & 9.85 ($-$6.1\%) & $+$0.33\% \\
B2 & 364 & 1.51 & 16.78 & 15.91 ($-$5.2\%) & $+$0.25\% \\
B3 & 651 & 4.40 & 16.78 & 15.91 ($-$5.2\%) & $+$0.09\% \\
B4 & 537 & 7.05 & 33.56 & 31.76 ($-$5.4\%) & $+$0.11\% \\
B5 & 537 & 13.21 & 33.56 & 31.76 ($-$5.4\%) & $+$0.06\% \\
B6 & 433 & 30.26 & 33.56 & 31.76 ($-$5.4\%) & $+$0.02\% \\
\bottomrule
\end{tabular}

\end{table}

\subsection{Compiled training}
\label{sec:train}

Training graphs have more freedom, because the backward pass interleaves gradient computation with the release of saved activations.
\Cref{tab:main} summarizes the savings of fixed-wrapper realization on CPU, and \cref{tab:full} in \cref{app:full} lists every configuration.

\begin{table}[t]
\centering\small
\caption{Peak saving of fixed-wrapper realization over stock Inductor (memory pass on), compiled training on CPU. The prediction error is 0\,B and gradients are bitwise identical in every run.}
\label{tab:main}
\setlength{\tabcolsep}{4pt}
\begin{tabular}{@{}lrrrr@{}}
\toprule
Group & $n$ & $>0$ & Median\,\% & Max\,\%\\
\midrule
Replication of earlier rows & \RreplicationN & \RreplicationPos & \RreplicationMedian & \RreplicationMax\\
Transformers (held-out) & \RtransformerN & \RtransformerPos & \RtransformerMedian & \RtransformerMax\\
CNNs (held-out) & \RcnnN & \RcnnPos & \RcnnMedian & \RcnnMax\\
BERT/GPT-2 batch$\times$sequence & \RsweepN & \RsweepPos & \RsweepMedian & \RsweepMax\\
\bottomrule
\end{tabular}
\end{table}

The saving comes from lifetimes at the critical backward call, not from allocation volume, which is identical in both programs.
\Cref{fig:timeline} shows the backward wrapper of a two-layer Transformer encoder.
The stock order materializes two parameter gradients, of 64\,KiB and 16\,KiB, before the flash-attention backward call.
The reordered program defers them past that call, holding one 32\,KiB activation a little longer in exchange.
The net saving is 48\,KiB, and the new peak equals the certified lower bound.
CNNs gain only where the backward pass has parallel branches (ResNet-50, ConvNeXt, Swin); chain-like ones are already optimal, and their certificates prove it.
The savings do not follow a simple rule.
Over the batch and sequence sweeps they are non-monotone, because the critical call changes with the shape.
\Cref{app:negative} gives two families in which either of the obvious rules, ``finish the gradient branch first'' or ``always defer gradients'', nearly doubles the peak.

\begin{figure}[t]
\centering
\includegraphics[width=\columnwidth]{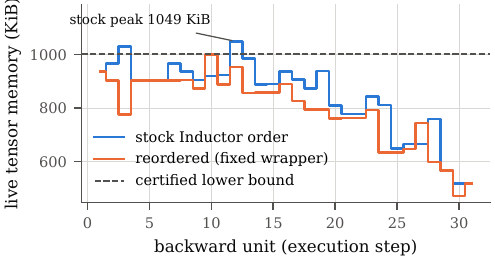}
\caption{Live memory at each unit of the backward wrapper of a two-layer Transformer encoder. Deferring two parameter-gradient materializations past the flash-attention backward call lowers the peak to the certified lower bound.}
\label{fig:timeline}
\end{figure}

\paragraph{Precise effects matter more than search.}
If every storage access is treated as a write, as a conservative implementation might do, most of the saving disappears.
The encoder loses all of its 4.6\%, the large-sequence encoder all of its 4.7\%, and DistilBERT all of its 2.3\%; the encoder--decoder drops from 7.6\% to 6.2\%.
The freedom lies in reordering readers of the same buffer generation, for example the two matrix multiplications of a linear layer's backward pass.
Disabling Inductor's own memory pass leaves the stock peak unchanged in six of eight configurations, and raises it by 1.3\% and 1.9\% in the other two.
From either starting point, fixed-wrapper realization ends below the stock program with the pass enabled.

\subsection{Certificates on held-out models}
\label{sec:certeval}

None of the models in this section were used to design the bounds.
On the \RCertAllN{} forward and backward wrapper graphs of the training grid and the inference grid, the trap bounds certify the found order optimal in \RCertAll{} cases (\RCertAllPct\%).
The forward bound alone certifies \RCertAllF{} and the reverse bound alone \RCertAllR, as \cref{fig:certificates}a shows.
In training the remaining gaps, between \RCertGapMin\% and \RCertGapMax\%, occur in the backward graphs of DistilBERT, BERT and ViT-B.
In inference there is a single open case, RegNetY-400MF, with a gap of \RInfGap\%, the only graph where the bounds are clearly loose.
These gaps may be due either to the bound or to the search.
As a baseline we posed the same infeasibility question ($\OPT-1$) to CP-SAT with an exact cumulative model (\cref{fig:certificates}b).
CP-SAT answers \RCpProved{} of \RCpN{} queries, with a median time of \RCpMedian\,s and a maximum of \RCpMax\,s, and times out on \RCpUnknown.
The certificates need a median of \RCertMedianCp\,s and at most \RCertMaxCp\,s, and unlike CP-SAT they can be checked independently.

\subsection{GPU replication and the caching allocator}
\label{sec:gpu}

We reran the training grid and the sweeps on the GPU (\RGpuN{} configurations).
The wrappers there call Triton kernels.
Correctness carries over.
In all \RGpuDetBit{} configurations where the stock program is itself deterministic, the reordered program is bitwise identical.
In the other \RGpuNondet{}, where cuDNN or the attention backward use atomics, the reordered gradients differ from stock by no more than stock differs from itself.
Exactness does not carry over (\cref{fig:gpu}).
Fixed-wrapper realization lowers the peak reported by the allocator in \RGpuPos{} configurations, by a median of \RGpuMedPos\% and up to \RGpuMax\%, but it raises it in \RGpuNeg, by up to \RGpuWorst\%.
The block-level prediction is exact in only \RGpuExact{} of \RGpuMod.

The cause is the frame condition (C5) of \cref{thm:lift}.
The CUDA caching allocator rounds requests, and in its large-block pool it does not split off remainders below 1\,MiB.
Which block a tensor receives, and therefore how many bytes it occupies, depends on the allocation history.
The logical peak, measured by the allocator's \texttt{requested\_bytes} counter, moves exactly as on CPU; the 8-layer Llama, for example, drops by 0.86\% in both settings.
Expandable segments do not remove the effect.
Replaying the CPU traces through a model of the caching allocator tells the same story.
Of the \RSimN{} configurations with a logical saving, reserved memory decreases in \RSimDown{}, stays the same in \RSimEq{} because 20\,MiB segments absorb the change, and increases in \RSimUp.
Because a single profiled step measures the true peak, we deploy reorderings on GPU behind a guard that keeps the stock program unless the measured peak improves.
With the guard no configuration regresses, and the mean saving is \RGpuGuardMean\%.
Static arenas, which embedded runtimes and ahead-of-time compilers use, do not have this problem.
There the allocation is a function of the schedule alone, so the condition holds by construction.

\begin{figure}[t]
\centering
\includegraphics[width=\columnwidth]{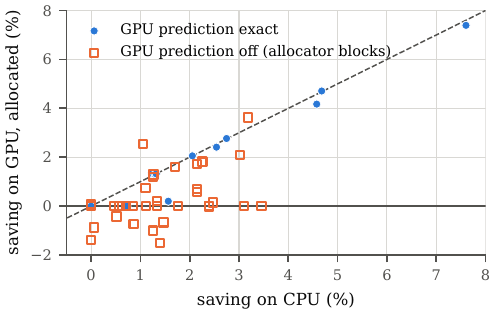}
\caption{Saving of the same reordering on CPU and on GPU (peak allocated by the CUDA caching allocator). Filled markers: the GPU peak equals the model prediction. Hollow squares: the allocator's history-dependent block sizes break the prediction.}
\label{fig:gpu}
\end{figure}

\subsection{Precision, pretrained weights and cost}
\label{sec:precision}

In bfloat16 inference (\RBfN{} configurations) fixed-wrapper realization is again exact and bitwise identical to stock in \RBfExact{} cases, and the certificates again prove the stock order optimal.
The schedule depends only on shapes and on the generated program, not on weight values.
With pretrained torchvision and Hugging Face checkpoints (\RPreN{} configurations), the stock and optimized peaks are identical to those with random weights in \RPreSame{} cases, and realization is exact in \RPreExact.
In particular, pretrained DeepLabV3 saves the same \RInfDeeplab\%.
One caveat concerns DeepLabV3's default torchvision checkpoint, which adds an auxiliary FCN head.
That head changes the graph, because an additional intermediate feature map must stay alive for it. This raises the peak by 10\% and makes the stock order optimal.
Saving memory therefore starts with deploying exactly the graph that is needed.
Reordering does not change the step time.
In an alternating protocol over 10 CPU configurations, the median reordered-to-stock time ratio is \RLatMedian, within [\RLatMin, \RLatMax].
On GPU the median ratio is \RGpuTimeMed.
Profiling, building the model and certification take one extra step plus under a second for almost every graph.

\section{Selective Recomputation and Its Limits}
\label{sec:remat}

When the certified optimum is still above the memory budget, the remaining lever is recomputation.
For real-time inference it is acceptable only in a selective form, where a few cheap values are recomputed so that a large buffer need not be kept.
We model this as in Moccasin~\cite{bartan2023moccasin}.
Vertex $v$ produces an immutable value of $w_v$ bytes at price $c_v$ (for example its MACs), and it may be computed at most $k_v$ times.
Its inputs must be resident when it runs, and its output is allocated before inputs are released.
The question is how little extra work buys a given memory budget $B$.
Uniform caps $k_v=k$ and individual caps express different policies.
An individual cap of two on cheap elementwise operators and one everywhere else is the kind of policy a real-time system can afford.

\paragraph{Certificates of necessary recomputation.}
The one-shot bounds of \cref{sec:cert} carry over.
If a capped schedule computed every vertex of a set $P$ exactly once with peak at most $B$, then the order of first computations would be a one-shot schedule of the graph with all weights outside $P$ set to zero, with peak at most $B$ (\cref{lem:proj}).
Hence any one-shot certificate above $B$ for that projected graph proves that some vertex of $P$ must be recomputed.
Constraints on individual events strengthen this into a lower bound on the extra work.
Every once-computed value that must survive an event contributes to that event's memory, which yields a knapsack-cover constraint on the recomputed bytes (\cref{app:remat}).
On \RTTwoN{} small DAGs with uniform cap two, where we compute everything exactly, three findings stand out.
\begin{itemize}
\item Recomputation lowers the optimal peak in \RTTwoHelps{} of them, by a mean factor of \RTTwoRatio.
\item Below the one-shot optimum, the projection certificate proves that recomputation is necessary in \RTTwoCert{} of \RTTwoB{} budgets.
\item The cover bound explains on average \RTTwoCover\% of the exact extra work.
\end{itemize}

\begin{table}[t]
\centering\small
\caption{Complexity of peak feasibility with recomputation caps (minimum work where noted). TTSP: two-terminal series--parallel graphs.}
\label{tab:complexity}
\begin{tabular}{@{}p{0.50\columnwidth}p{0.40\columnwidth}@{}}
\toprule
Restriction & Result \\
\midrule
fixed uniform cap $k\ge1$, treewidth $\le2$ & strongly NP-complete\\
caps in binary & PSPACE-complete (classical)\\
min.\ work, uniform $k{=}2$, TTSP, out-degree $\le2$ & equals 0--1 knapsack\\
caps $\le3$, TTSP, out-degree $\le 3$ & NP-complete\\
(peak, work) frontier, $k{=}2$, TTSP & $2^m$ points on $O(m)$ vertices\\
min.\ work, tree decomposition of width $t$ & $f(t,K,B)\,\poly(n,L)$\\
\bottomrule
\end{tabular}
\end{table}

\paragraph{Complexity.}
\Cref{tab:complexity} summarizes what can and cannot be computed efficiently.

\begin{theorem}\label{thm:hard}
For every fixed $k\ge1$, deciding whether a DAG with uniform recomputation cap $k$ can compute its sink with peak at most $B$ is strongly NP-complete, even on DAGs of treewidth at most two.
\end{theorem}

For $k=1$ this is the strong NP-hardness of weighted one-shot scheduling on pumpkin graphs, which have treewidth two~\cite{jin2023newtools}.
For larger $k$ we append a reset gadget: $k$ phases separated by barriers of weight $B$, which force every old value to be evicted.
With a uniform cap of $k$ each old vertex then runs exactly once per phase, so each phase is a one-shot schedule.
The gadget is a fan at the sink and keeps the treewidth at two (\cref{app:hard}).
Small interfaces do not help either.
On series--parallel graphs with out-degree two and a uniform cap of two, the minimum extra work is \emph{exactly} the optimum of a knapsack instance,
\[
\OPT(B)-\sum_v c_v=\sum_i c_i-\max\Big\{\sum_{i\in S}c_i:\sum_{i\in S}w_i\le C\Big\},
\]
and the trade-off between peak and work can have exponentially many non-dominated points.

\paragraph{An exact algorithm on bounded treewidth.}
\begin{theorem}\label{thm:fpt}
Given a tree decomposition of width $t$ and an integer budget $B$, the minimum total work subject to peak at most $B$ and all caps can be computed in time $f(t,K,B)\poly(n,L)$, where $K$ bounds the caps and $L$ is the input length.
\end{theorem}

The difficulty is that the obvious interface state, namely which boundary values are resident and how often each was computed, is not sufficient.
Two histories can agree on it and still differ in which \emph{internal} vertices have exhausted their caps (\cref{app:fpt}).
The algorithm therefore records, at each node of the decomposition, the complete word of compute and delete events of the bag vertices, of length at most $2K(t+1)$.
Between consecutive boundary events it keeps the typical-sequence reduction~\cite{bodlaender2020typical} of the private memory profile, of length at most $2B+1$, together with a literal witness schedule.
\Cref{thm:hard} shows that the dependence on $B$ cannot be reduced to $\poly(B)$ unless P$=$NP.
A rounding step yields a schedule with peak at most $(1+\varepsilon)B$ and work at most that of any schedule with at most $p$ live values (\cref{app:fpt}).
The algorithm is a statement about tractability, not a tool.
Our prototype exceeds $3\cdot10^6$ states on width-two graphs with six vertices at $K=2$ and $B=8$ (\cref{app:t3}).

\paragraph{What this means for tiny models.}
On the public microcontroller-class models of \cref{sec:tiny}, recomputation of whole tensors is useless, because the peak is a single operator's footprint.
Tiling that operator, and possibly its neighbours, turns one large vertex into many small ones that share halo regions.
Recomputing a halo is then a cap $k_v>1$ on a small vertex, and the extra work is the halo's MACs.
Tile-level graphs of a few early layers are narrow, so the exact machinery of this section, together with the certificates of \cref{sec:cert}, applies to them directly.
Choosing tile shapes jointly with the schedule, under an explicit latency budget and with a certificate of optimality, is the natural continuation of this work.
Patch-based inference today chooses them heuristically~\cite{lin2021mcunetv2}.

\section{Related Work}
\label{sec:related}

\paragraph{Peak-memory scheduling.}
Register sufficiency, the unit-weight version of the problem, is NP-complete~\cite{sethi1975complete}.
Jin et al.\ prove strong NP-hardness of the weighted one-shot problem on pumpkins and out-trees, and give exact algorithms for series--parallel graphs based on dominant schedules~\cite{jin2023newtools}.
SERENITY schedules irregularly wired networks for edge devices by dynamic programming and graph rewriting~\cite{ahn2020serenity}.
HMCOS expands hierarchical groups around the peak~\cite{wang2022hmcos}.
MODeL optimizes order and address placement jointly with an ILP~\cite{steiner2023model}, and Zhong et al.\ fold safe subgraphs before solving an ILP~\cite{zhong2023iccad}.
Lower bounds based on event cuts and wavefronts are standard in I/O complexity~\cite{elango2014}, and iterated optimistic admission against resource bounds is the LP--RPG pattern in planning~\cite{coles2013lprpg}.
Our trap bounds specialize these ideas to tensor roots.
To our knowledge, their domination over the event-cut bound, their incomparability and their computation with $O(n^2)$ cuts have not been stated before.

\paragraph{Memory planning in compilers and embedded runtimes.}
XLA's heap simulator accounts for allocation, free and sharing events over an instruction sequence~\cite{xla-heap}.
MLIR provides staged side effects and ownership-based deallocation~\cite{mlir}.
Exo verifies schedules of low-level code against memory effects~\cite{ikarashi2022exo}.
$\mu$-TWO profiles memory and executes optimized schedules for multi-model training~\cite{purandare2023mutwo}.
TFLite Micro executes operators in model order over a statically planned arena~\cite{david2021tflm}.
MCUNet co-designs architectures and inference for microcontrollers~\cite{lin2020mcunet}, and MCUNetV2 executes early layers patch by patch, recomputing overlaps~\cite{lin2021mcunetv2}, as fused-layer accelerators do in hardware~\cite{alwani2016fused}.
Sun et al.\ study recomputation for peak-memory reduction of inference on IoT devices~\cite{sun2023iot}.
Our contributions are the measurement that regeneration is unreliable, a realization that is exact on the generated program, and certificates that say when scheduling is finished.

\paragraph{Rematerialization.}
Checkpointing for chains goes back to revolve~\cite{griewank2000revolve} and to sublinear-memory training~\cite{chen2016training}.
Heterogeneous chains admit optimal dynamic programs~\cite{beaumont2019optimal}.
Checkmate~\cite{jain2020checkmate}, Moccasin~\cite{bartan2023moccasin}, Rockmate~\cite{zhao2023rockmate}, HiRemate~\cite{gusak2025hiremate}, MAGIS~\cite{chen2024magis} and DTR~\cite{kirisame2021dtr} optimize recomputation with MILP, constraint programming, hierarchy, graph transformations or online policies.
Kumar et al.\ bound recomputation through treewidth~\cite{kumar2019efficient}.
Bounded-width composition and typical sequences are classical~\cite{downey1999parameterized,bodlaender1996efficient,bodlaender2020typical}, and unrestricted pebbling is PSPACE-complete~\cite{gilbert1980pebbling,chan2023pebbling}.
Our complexity results delimit what selective recomputation with caps can achieve, which is the form relevant to real-time inference.

\section{Limitations}
\label{sec:limits}

Exactness, as stated in \cref{thm:lift}, holds for allocators that return exactly what is requested, which is the case on our CPU runs and in static arenas.
For the CUDA caching allocator we fall back to a measured guard.
Searching against a simulated caching allocator could restore exactness, since its behaviour is deterministic for a fixed single-stream request sequence, but we have not yet matched the simulator to the device, whose initial state includes the segments holding parameters.
Library kernels' effects and allocation sequences are profiled and checked by re-measurement, not proven, and every model is specific to one input shape.
The tiny-model study evaluates the operator graph and a greedy arena planner, not firmware on a microcontroller.
The savings from reordering are small except on branching networks, typically 1--3\% in training and zero on most inference graphs.
We present this as a finding that the certificates make visible, not as a weakness of the search.
\Cref{thm:fpt} is impractical, its novelty relative to general finite-state pebbling frameworks was checked only by a bounded literature review, and whether $\OPT/\max(\mathcal F,\mathcal R)$ is bounded remains open.

\section{Conclusion}

For memory-constrained inference, the cheapest memory is the memory saved by a free schedule change, and the most useful answer a scheduler can give is a certified one.
We showed that a schedule must be realized on the program that runs: regenerating code for a better order is unreliable, while permuting the generated program under a checked contract is exact.
Polynomial-time certificates show that most inference graphs, and all the public microcontroller-class models we studied, are already scheduled optimally, and they identify the branching networks where reordering pays.
Beyond the certified optimum, selective recomputation with caps is computationally hard in every general sense we examined.
It is useless on public tiny models whose peak is a single layer, but on architectures that hold a cheap value for a late reuse it pays: \RPubMin--\RPubMax\% on ReDimNet at under 0.4\% extra MACs.
This points to certified tile-level schedules as the way to buy memory at a bounded, explicit cost in latency.

\section*{Generative AI usage statement}
In this work, we used generative AI tools to assist with translation (parts of the text were translated from the authors' draft in another language into English), to implement methods (the experimental code, including the compiler harness, the baselines and the scripts that produce the tables and figures from the saved results), and to generate synthetic data sets (the random DAG families of \cref{app:synthetic}, whose construction and properties were checked by the authors).
Supporting qualitative and thematic data analysis is not applicable to this work.
Additionally, we used generative AI tools to polish the writing, to create the figures from our experimental results, to find related work, and to check that bibliographic entries exist and are cited correctly.
We have reviewed all AI-assisted work: all code and all text were checked line by line by the authors.
We take responsibility for the final content of this work, including text and claims produced with the aid of generative AI.

\bibliographystyle{mlsys2025}
\bibliography{refs}

@inproceedings{ansel2024pytorch2,
  title     = {{PyTorch 2}: Faster Machine Learning Through Dynamic {Python} Bytecode Transformation and Graph Compilation},
  author    = {Ansel, Jason and Yang, Edward and He, Horace and Gimelshein, Natalia and Jain, Animesh and Voznesensky, Michael and Bao, Bin and Bell, Peter and Berard, David and Burovski, Evgeni and Chauhan, Geeta and Chourdia, Anjali and Constable, Will and Desmaison, Alban and DeVito, Zachary and Ellison, Elias and Feng, Will and Gong, Jiong and Gschwind, Michael and Hirsh, Brian and Huang, Sherlock and Kalambarkar, Kshiteej and Kirsch, Laurent and Lazos, Michael and Lezcano, Mario and Liang, Yanbo and Liang, Jason and Lu, Yinghai and Luk, C. K. and Maher, Bert and Pan, Yunjie and Puhrsch, Christian and Reso, Matthias and Saroufim, Mark and Siraichi, Marcos Yukio and Suk, Helen and Zhang, Shunting and Suo, Michael and Tillet, Phil and Zhao, Xu and Wang, Eikan and Zhou, Keren and Zou, Richard and Wang, Xiaodong and Mathews, Ajit and Wen, William and Chanan, Gregory and Wu, Peng and Chintala, Soumith},
  booktitle = {ASPLOS},
  pages     = {929--947},
  year      = {2024}
}

@misc{xla-heap,
  title        = {{XLA} heap simulator},
  author       = {{OpenXLA}},
  howpublished = {\url{https://github.com/openxla/xla/blob/main/xla/service/heap_simulator/heap_simulator.h}},
  year         = {2026}
}

@misc{mlir,
  title        = {{MLIR}: Side Effects and Speculation; Ownership-Based Buffer Deallocation},
  author       = {{LLVM Project}},
  howpublished = {\url{https://mlir.llvm.org/docs/}},
  year         = {2026}
}

@inproceedings{wang2022hmcos,
  title     = {Hierarchical Memory-Constrained Operator Scheduling of Neural Architecture Search Networks},
  author    = {Wang, Zihan and Wan, Chengcheng and Chen, Yuting and Lin, Ziyi and Jiang, He and Qiao, Lei},
  booktitle = {Design Automation Conference (DAC)},
  year      = {2022}
}

@inproceedings{steiner2023model,
  title     = {{MODeL}: Memory Optimizations for Deep Learning},
  author    = {Steiner, Benoit and Elhoushi, Mostafa and Kahn, Jacob and Hegarty, James},
  booktitle = {ICML},
  pages     = {32618--32632},
  year      = {2023}
}

@article{jin2023newtools,
  title   = {New Tools for Peak Memory Scheduling},
  author  = {Jin, Ce and Purohit, Manish and Svitkina, Zoya and Vee, Erik and Wang, Joshua R.},
  journal = {arXiv:2312.13526},
  year    = {2023}
}

@inproceedings{zhong2023iccad,
  title     = {Memory-aware Scheduling for Complex Wired Networks with Iterative Graph Optimization},
  author    = {Zhong, Shuzhang and Li, Meng and Liang, Yun and Wang, Runsheng and Huang, Ru},
  booktitle = {ICCAD (arXiv:2308.13898)},
  year      = {2023}
}

@inproceedings{elango2014,
  title     = {On Characterizing the Data Movement Complexity of Computational {DAGs} for Parallel Execution},
  author    = {Elango, Venmugil and Rastello, Fabrice and Pouchet, Louis-No{\"e}l and Ramanujam, J. and Sadayappan, P.},
  booktitle = {SPAA},
  year      = {2014}
}

@article{coles2013lprpg,
  title   = {A Hybrid {LP-RPG} Heuristic for Modelling Numeric Resource Flows in Planning},
  author  = {Coles, Amanda and Coles, Andrew and Fox, Maria and Long, Derek},
  journal = {Journal of Artificial Intelligence Research},
  volume  = {46},
  year    = {2013}
}

@article{sethi1975complete,
  title   = {Complete Register Allocation Problems},
  author  = {Sethi, Ravi},
  journal = {SIAM Journal on Computing},
  volume  = {4},
  number  = {3},
  pages   = {226--248},
  year    = {1975}
}

@inproceedings{purandare2023mutwo,
  title     = {$\mu$-{TWO}: 3$\times$ Faster Multi-Model Training with Orchestration and Memory Optimization},
  author    = {Purandare, Sanket and Wasay, Abdul and Jain, Animesh and Idreos, Stratos},
  booktitle = {MLSys},
  year      = {2023}
}

@inproceedings{ikarashi2022exo,
  title     = {Exocompilation for Productive Programming of Hardware Accelerators},
  author    = {Ikarashi, Yuka and Bernstein, Gilbert Louis and Reinking, Alex and Genc, Hasan and Ragan-Kelley, Jonathan},
  booktitle = {PLDI},
  pages     = {703--718},
  year      = {2022}
}

@article{griewank2000revolve,
  title   = {Algorithm 799: Revolve: An Implementation of Checkpointing for the Reverse or Adjoint Mode of Computational Differentiation},
  author  = {Griewank, Andreas and Walther, Andrea},
  journal = {ACM Transactions on Mathematical Software},
  volume  = {26},
  number  = {1},
  pages   = {19--45},
  year    = {2000}
}

@article{chen2016training,
  title   = {Training Deep Nets with Sublinear Memory Cost},
  author  = {Chen, Tianqi and Xu, Bing and Zhang, Chiyuan and Guestrin, Carlos},
  journal = {arXiv:1604.06174},
  year    = {2016}
}

@article{beaumont2019optimal,
  title   = {Optimal Checkpointing for Heterogeneous Chains: How to Train Deep Neural Networks with Limited Memory},
  author  = {Herrmann, Julien and Beaumont, Olivier and Eyraud-Dubois, Lionel and Joly, Alexis and Shilova, Alena},
  journal = {arXiv:1911.13214},
  year    = {2019}
}

@inproceedings{jain2020checkmate,
  title     = {Checkmate: Breaking the Memory Wall with Optimal Tensor Rematerialization},
  author    = {Jain, Paras and Jain, Ajay and Nrusimha, Aniruddha and Gholami, Amir and Abbeel, Pieter and Keutzer, Kurt and Stoica, Ion and Gonzalez, Joseph E.},
  booktitle = {MLSys},
  year      = {2020}
}

@inproceedings{bartan2023moccasin,
  title     = {Moccasin: Efficient Tensor Rematerialization for Neural Networks},
  author    = {Bartan, Burak and Li, Haoming and Teague, Harris and Lott, Christopher and Dilkina, Bistra},
  booktitle = {ICML},
  year      = {2023}
}

@inproceedings{zhao2023rockmate,
  title     = {Rockmate: an Efficient, Fast, Automatic and Generic Tool for Re-materialization in {PyTorch}},
  author    = {Zhao, Xunyi and Le Hellard, Th{\'e}otime and Eyraud-Dubois, Lionel and Gusak, Julia and Beaumont, Olivier},
  booktitle = {ICML},
  year      = {2023}
}

@inproceedings{gusak2025hiremate,
  title     = {{HiRemate}: Hierarchical Approach for Efficient Re-materialization of Neural Networks},
  author    = {Gusak, Julia and Zhao, Xunyi and Le Hellard, Th{\'e}otime and Li, Zhe and Eyraud-Dubois, Lionel and Beaumont, Olivier},
  booktitle = {ICML},
  year      = {2025}
}

@inproceedings{chen2024magis,
  title     = {{MAGIS}: Memory Optimization via Coordinated Graph Transformation and Scheduling for {DNN}},
  author    = {Chen, Renze and Ding, Zijian and Zheng, Size and Zhang, Chengrui and Leng, Jingwen and Liu, Xuanzhe and Liang, Yun},
  booktitle = {ASPLOS},
  year      = {2024}
}

@inproceedings{kirisame2021dtr,
  title     = {Dynamic Tensor Rematerialization},
  author    = {Kirisame, Marisa and Lyubomirsky, Steven and Haan, Altan and Brennan, Jennifer and He, Mike and Roesch, Jared and Chen, Tianqi and Tatlock, Zachary},
  booktitle = {ICLR},
  year      = {2021}
}

@inproceedings{kumar2019efficient,
  title     = {Efficient Rematerialization for Deep Networks},
  author    = {Kumar, Ravi and Purohit, Manish and Svitkina, Zoya and Vee, Erik and Wang, Joshua R.},
  booktitle = {NeurIPS},
  pages     = {15146--15155},
  year      = {2019}
}

@book{downey1999parameterized,
  title     = {Parameterized Complexity},
  author    = {Downey, Rodney G. and Fellows, Michael R.},
  publisher = {Springer},
  year      = {1999}
}

@article{bodlaender1996efficient,
  title   = {Efficient and Constructive Algorithms for the Pathwidth and Treewidth of Graphs},
  author  = {Bodlaender, Hans L. and Kloks, Ton},
  journal = {Journal of Algorithms},
  volume  = {21},
  number  = {2},
  pages   = {358--402},
  year    = {1996}
}

@inproceedings{bodlaender2020typical,
  title     = {Typical Sequences Revisited --- Computing Width Parameters of Graphs},
  author    = {Bodlaender, Hans L. and Jaffke, Lars and Telle, Jan Arne},
  booktitle = {STACS},
  year      = {2020}
}

@article{gilbert1980pebbling,
  title   = {The Pebbling Problem is Complete in Polynomial Space},
  author  = {Gilbert, John R. and Lengauer, Thomas and Tarjan, Robert Endre},
  journal = {SIAM Journal on Computing},
  volume  = {9},
  number  = {3},
  pages   = {513--524},
  year    = {1980}
}

@article{chan2023pebbling,
  title   = {Hardness of Approximation in {PSPACE} and Separation Results for Pebble Games},
  author  = {Chan, Siu Man and Lauria, Massimo and Nordstr{\"o}m, Jakob and Vinyals, Marc},
  journal = {arXiv:2305.19104},
  year    = {2023}
}

@inproceedings{lin2020mcunet,
  title     = {{MCUNet}: Tiny Deep Learning on {IoT} Devices},
  author    = {Lin, Ji and Chen, Wei-Ming and Lin, Yujun and Cohn, John and Gan, Chuang and Han, Song},
  booktitle = {NeurIPS},
  year      = {2020}
}

@inproceedings{lin2021mcunetv2,
  title     = {{MCUNetV2}: Memory-Efficient Patch-based Inference for Tiny Deep Learning},
  author    = {Lin, Ji and Chen, Wei-Ming and Cai, Han and Gan, Chuang and Han, Song},
  booktitle = {NeurIPS},
  year      = {2021}
}

@inproceedings{david2021tflm,
  title     = {{TensorFlow Lite Micro}: Embedded Machine Learning for {TinyML} Systems},
  author    = {David, Robert and Duke, Jared and Jain, Advait and Janapa Reddi, Vijay and Jeffries, Nat and Li, Jian and Kreeger, Nick and Nappier, Ian and Natraj, Meghna and Regev, Shlomi and Rhodes, Rocky and Wang, Tiezhen and Warden, Pete},
  booktitle = {MLSys},
  year      = {2021}
}

@inproceedings{ahn2020serenity,
  title     = {Ordering Chaos: Memory-Aware Scheduling of Irregularly Wired Neural Networks for Edge Devices},
  author    = {Ahn, Byung Hoon and Lee, Jinwon and Lin, Jamie Menjay and Cheng, Hsin-Pai and Hou, Jilei and Esmaeilzadeh, Hadi},
  booktitle = {MLSys},
  year      = {2020}
}

@inproceedings{banbury2021mlperftiny,
  title     = {{MLPerf Tiny} Benchmark},
  author    = {Banbury, Colby and Janapa Reddi, Vijay and Torelli, Peter and Holleman, Jeremy and Jeffries, Nat and Kiraly, Csaba and Montino, Pietro and Kanter, David and Ahmed, Sebastian and Pau, Danilo and Thakker, Urmish and Torrini, Antonio and Warden, Peter and Cordaro, Jay and Di Guglielmo, Giuseppe and Duarte, Javier and Gibellini, Stephen and Parekh, Videet and Tran, Honson and Tran, Nhan and Wenxu, Niu and Xuesong, Xu},
  booktitle = {NeurIPS Datasets and Benchmarks},
  year      = {2021}
}

@article{zhang2017helloedge,
  title   = {Hello Edge: Keyword Spotting on Microcontrollers},
  author  = {Zhang, Yundong and Suda, Naveen and Lai, Liangzhen and Chandra, Vikas},
  journal = {arXiv:1711.07128},
  year    = {2017}
}

@inproceedings{choi2019tcresnet,
  title     = {Temporal Convolution for Real-time Keyword Spotting on Mobile Devices},
  author    = {Choi, Seungwoo and Seo, Seokjun and Shin, Beomjun and Byun, Hyeongmin and Kersner, Martin and Kim, Beomsu and Kim, Dongyoung and Ha, Sungjoo},
  booktitle = {Interspeech},
  year      = {2019}
}

@inproceedings{alwani2016fused,
  title     = {Fused-layer {CNN} Accelerators},
  author    = {Alwani, Manoj and Chen, Han and Ferdman, Michael and Milder, Peter},
  booktitle = {MICRO},
  pages     = {1--12},
  year      = {2016}
}

@inproceedings{sun2023iot,
  title     = {Minimizing Peak Memory Footprint of Inference on {IoT} Devices by Efficient Recomputation},
  author    = {Sun, Xiaofeng and Xu, Chaonong and Li, Chao},
  booktitle = {Advanced Intelligent Computing Technology and Applications (ICIC), LNCS},
  pages     = {15--26},
  year      = {2023}
}

@inproceedings{yakovlev2024redimnet,
  title     = {Reshape Dimensions Network for Speaker Recognition},
  author    = {Yakovlev, Ivan and Makarov, Rostislav and Balykin, Andrei and Malov, Pavel and Okhotnikov, Anton and Torgashov, Nikita},
  booktitle = {Interspeech},
  year      = {2024}
}

\appendix
\onecolumn

\section*{Appendix overview}

The appendix is organized in four parts.
\Cref{app:setup} describes the experimental setup in enough detail to reproduce every number in the paper.
\Cref{app:results} reports the complete results behind the summaries of \cref{sec:eval}, including the configurations that appear only in aggregate in the main text.
\Cref{app:lift,app:traps,app:remat,app:hard,app:fpt} contain the proofs, in the order in which the results appear in the paper.
\Cref{app:artifact} describes the supplementary archive.

\section{Experimental setup}
\label{app:setup}

\subsection{Software and hardware}
All CPU experiments use PyTorch~2.14.1 (CPU build), Hugging Face Transformers~4.57.1 and torchvision~0.29.1 on a single thread (\texttt{OMP\_NUM\_THREADS=1}) of a 22-core x86-64 workstation.
GPU experiments use the CUDA~12.6 build of the same PyTorch release on an NVIDIA GeForce RTX 2080 Ti (Turing, 11\,GB), one GPU per configuration, with no other process on the device.
CP-SAT is the OR-Tools implementation with one worker.
Weights are random (seed 1802) unless stated otherwise, inputs are synthetic, and dropout and stochastic depth are disabled, because both make even the stock program non-reproducible.

\subsection{Compiled-program experiments}
\paragraph{What is measured.}
For training we compile a function that runs the forward pass and an MSE loss against a fixed random target, and then call \texttt{backward()}.
There is no optimizer step.
For inference we compile the model in evaluation mode and run it under \texttt{torch.no\_grad()}.
The stock baseline is Inductor with \texttt{reorder\_for\_peak\_memory} and \texttt{inplace\_buffers} enabled, which is the default; only the FX-graph cache is disabled, so that every run compiles afresh.
The peak is the largest value of the profiler's tracked tensor-allocation counter during one warm step.
It includes the step's outputs, its loss and all gradients, and excludes the resident weights and inputs.
The column ``including the frame'' of \cref{tab:full} adds the bytes of the parameters, buffers, inputs and target back to the denominator.

\paragraph{Protocol per configuration.}
Each configuration runs in a fresh process and goes through the following steps.
\begin{enumerate}
\item Compile once with the stock compiler and check the compiled outputs and gradients against eager execution ($\mathrm{rtol}=10^{-4}$, $\mathrm{atol}=10^{-5}$).
\item Run two warm steps, then one observed step in which every wrapper statement executes inside a profiler scope and the storages bound to its assigned names are recorded.
\item Build the root model of every wrapper, compute the trap bounds (20\,s limit each) and search for an order (60\,s limit).
\item Install three versions of every wrapper in the same compiled module: the stock body, the stock unit order re-emitted by our emitter (``normalized''), and the optimized order.
\item Measure the three versions twice each, alternating, and compare their outputs, losses and gradients bitwise.
\item Force the optimized order into Inductor's reordering pass, recompile, and measure the regenerated program (``Regen.'').
\end{enumerate}
The prediction error of a wrapper is the measured peak inside the wrapper's window minus the predicted one.
The prediction is the stock entry level plus the model's peak minus the model's initial live bytes.

\paragraph{Configurations.}
Transformer encoders and encoder--decoders are \texttt{torch.nn} modules with four heads and a feed-forward width of four times the model width unless stated otherwise.
Hugging Face models use their default configurations with the stated numbers of layers.
For BERT, DistilBERT, GPT-2, T5 and ViT these are the base or small sizes; the Llama models are scaled down with hidden sizes 512 and 768.
Vision models are torchvision definitions.
The labels in the tables give the batch size $b$, the sequence length $s$, the model width $d$ and the number of layers.

\paragraph{Latency.}
Each latency configuration alternates the stock and the reordered program over four rounds.
Every round runs three warm-up steps and then 5 to 50 timed steps per variant, depending on the model's size.
We report the ratio of the medians.

\subsection{Operator-level inference experiments}
For the microcontroller-class models (\cref{sec:tiny}), the speaker-embedding networks (\cref{sec:rdn}) and the YOLO detectors we trace the model at the operator level and build an operator graph the way an embedded runtime sees it.
\begin{itemize}
\item \textbf{Tracing.} Small convolutional models are traced symbolically, with batch normalization folded into convolutions. The others are traced with \texttt{make\_fx} on concrete shapes after functionalization, decomposing to core ATen operators but keeping activations as single operators.
\item \textbf{Fusion.} Elementwise operators with a single tensor input and the same number of elements (activations, decomposed normalization) are fused into their producer when they are its only user.
\item \textbf{Views.} View operators alias their source and own no buffer.
\item \textbf{Sizes.} Activations are int8 for the microcontroller models and fp32 for the speaker and detection models; each surviving operator owns one output buffer.
\item \textbf{Outputs.} The model outputs are pinned until the end, so multi-output detectors are modeled correctly.
\end{itemize}
A schedule is executed operator by operator in Python, keeping exactly the values the schedule keeps.
The executed peak and the bitwise equality of all outputs with an unscheduled reference run are checked for every schedule we report.
The arena is computed with a TFLite-Micro-style greedy planner, which places buffers by decreasing size at the lowest offset that does not conflict, with 16- or 64-byte alignment.
MACs are counted with PyTorch's FLOP counter.

\subsection{Selective recomputation}
Candidates for recomputation are values that are produced before the peak event of the current optimal order and have a consumer after it.
For each candidate we enumerate single anchors $a$ among its ancestors whose cone, i.e.\ the operators between $a$ and the candidate, is closed: every input of an operator in the cone is $a$ or lies in the cone.
We keep the anchor that saves the most bytes, preferring cheaper cones on ties.
A transformation duplicates the cone, either once for all late consumers or once per late consumer, and rewires those consumers to the copy.
Every transformed graph is scheduled exactly and certified with the trap bounds.
A greedy loop applies the transformation with the lowest resulting peak as long as the peak decreases and the extra MACs stay within the budget.

\subsection{Theory experiments}
\label{app:synthetic}
The random DAG families of \cref{app:t1} have between 8 and 15 vertices with integer weights between 1 and 8.
\begin{itemize}
\item \textbf{random}: edges with probability 0.15--0.45 to every earlier vertex;
\item \textbf{sp}: random two-terminal series--parallel graphs;
\item \textbf{pumpkin}: two to four chains between a source and a sink;
\item \textbf{intree}: random in-trees;
\item \textbf{training}: a forward chain followed by a backward chain that reads the saved forward activations.
\end{itemize}
Exact optima are computed by dynamic programming over all ideals.
The capped experiments use exact Dijkstra search over (live set, counters) states with uniform cap two.

\section{Complete results}
\label{app:results}

\subsection{Compiled training on CPU}
\label{app:full}
\Cref{tab:full} lists every configuration of the main training grid.
All 78 training runs, which add the shape sweeps, ablations, seed repeats, latency runs and runs without the Inductor pass, have zero prediction error and bitwise-identical outputs and gradients.

\begin{table}[h]
\centering\scriptsize
\caption{Compiled training on CPU, main grid. Saving is relative to stock Inductor with its memory pass. The prediction error is the largest absolute error over the wrapper windows of the fixed-wrapper run. ``Regen.'' hands the same order back to Inductor. ``Cert.\ opt.'' counts wrapper graphs for which $\max(\mathcal F,\mathcal R)$ equals the found peak.}
\label{tab:full}
\resizebox{\textwidth}{!}{\begin{tabular}{@{}lrrrrrrrc@{}}
\toprule
Configuration & Params (M) & Stock Inductor (B) & Fixed wrapper (B) & Saving \% & Pred.\ err & Regen.\ saving \% & Regen.\ err (B) & Cert.\ opt. \\
\midrule
TransformerEnc 2L $d$64 $s$128 & 0.1 & 1{,}074{,}056 & 1{,}024{,}904 & 4.58 & 0 & 4.58 & 0 & 2/2 \\
TransformerEnc 4L $d$64 $s$128 & 0.2 & 1{,}932{,}168 & 1{,}883{,}016 & 2.54 & 0 & 2.54 & 0 & 2/2 \\
TransformerEnc 2L $d$256 $s$128 & 1.6 & 6{,}842{,}376 & 6{,}842{,}376 & 0.00 & 0 & 0.00 & 0 & 2/2 \\
TransformerEnc 2L $d$256 $s$256 & 1.6 & 8{,}678{,}408 & 8{,}542{,}472 & 1.57 & 0 & 1.57 & 0 & 2/2 \\
TransformerEnc 2L $d$256 $s$512 & 1.6 & 16{,}813{,}320 & 16{,}026{,}888 & 4.68 & 0 & 4.68 & 0 & 2/2 \\
TransformerEnc 2L $d$256 $s$1024 & 1.6 & 32{,}557{,}064 & 31{,}520{,}776 & 3.18 & 0 & 3.18 & 0 & 2/2 \\
Transformer enc--dec $d$64 & 0.2 & 1{,}073{,}928 & 992{,}264 & 7.60 & 0 & 6.13 & 15{,}872 & 2/2 \\
DistilBERT $b$1 $s$128 & 66.4 & 266{,}628{,}104 & 266{,}628{,}104 & 0.00 & 0 & -0.14 & 385{,}536 & 2/2 \\
DistilBERT $b$4 $s$512 & 66.4 & 657{,}635{,}336 & 641{,}931{,}272 & 2.39 & 0 & 2.39 & 0 & 2/2 \\
TransformerEnc 6L $d$256 $s$256 & 4.7 & 22{,}475{,}016 & 22{,}198{,}536 & 1.23 & 0 & 1.23 & 0 & 2/2 \\
TransformerEnc 4L $d$512 $s$512 & 12.6 & 62{,}243{,}080 & 61{,}182{,}216 & 1.70 & 0 & 1.70 & 0 & 2/2 \\
Transformer enc--dec $d$256 & 3.7 & 15{,}142{,}920 & 14{,}946{,}312 & 1.30 & 0 & 0.00 & 196{,}608 & 2/2 \\
DistilBERT $b$2 $s$256 & 66.4 & 278{,}829{,}064 & 272{,}537{,}608 & 2.26 & 0 & 1.70 & 1{,}564{,}672 & 1/2 \\
BERT-base $b$1 $s$128 & 109.5 & 436{,}743{,}176 & 436{,}743{,}176 & 0.00 & 0 & -0.09 & 385{,}536 & 2/2 \\
BERT-base $b$2 $s$256 & 109.5 & 446{,}582{,}792 & 441{,}854{,}984 & 1.06 & 0 & 1.06 & 3{,}072 & 1/2 \\
GPT-2 small $b$1 $s$256 & 124.4 & 503{,}269{,}384 & 500{,}904{,}968 & 0.47 & 0 & 0.47 & 0 & 2/2 \\
GPT-2 small $b$2 $s$512 & 124.4 & 712{,}161{,}288 & 699{,}590{,}664 & 1.77 & 0 & 1.77 & 0 & 2/2 \\
Llama 4L $h$512 $b$1 $s$256 & 29.0 & 119{,}296{,}008 & 117{,}721{,}096 & 1.32 & 0 & 1.32 & 0 & 2/2 \\
Llama 4L $h$512 $b$2 $s$512 & 29.0 & 151{,}750{,}664 & 149{,}852{,}424 & 1.25 & 0 & 1.25 & 0 & 2/2 \\
Llama 8L $h$768 $b$1 $s$512 & 81.2 & 334{,}289{,}928 & 331{,}402{,}248 & 0.86 & 0 & 0.86 & 0 & 2/2 \\
T5-small enc $b$2 $s$128 & 35.3 & 144{,}471{,}048 & 143{,}420{,}424 & 0.73 & 0 & -3.72 & 6{,}422{,}528 & 2/2 \\
T5-small enc $b$1 $s$512 & 35.3 & 153{,}120{,}776 & 149{,}975{,}048 & 2.05 & 0 & -30.65 & 50{,}072{,}576 & 2/2 \\
ViT-S/16 (6L) $b$4 & 11.0 & 129{,}562{,}456 & 126{,}002{,}248 & 2.75 & 0 & 2.75 & 0 & 2/2 \\
ViT-B/16 $b$2 & 85.8 & 361{,}024{,}904 & 352{,}079{,}240 & 2.48 & 0 & 2.48 & 0 & 1/2 \\
ResNet-18 $b$8 & 11.7 & 213{,}121{,}800 & 213{,}121{,}800 & 0.00 & 0 & 0.00 & 0 & 2/2 \\
ResNet-50 $b$4 & 25.6 & 418{,}362{,}888 & 408{,}368{,}648 & 2.39 & 0 & 1.14 & 5{,}218{,}304 & 2/2 \\
MobileNetV2 $b$8 & 3.5 & 436{,}370{,}056 & 436{,}370{,}056 & 0.00 & 0 & -1.20 & 5{,}246{,}848 & 2/2 \\
EfficientNet-B0 $b$4 & 5.3 & 224{,}117{,}064 & 224{,}117{,}064 & 0.00 & 0 & -3.10 & 6{,}949{,}856 & 2/2 \\
DenseNet-121 $b$2 & 8.0 & 261{,}278{,}600 & 261{,}278{,}600 & 0.00 & 0 & -2.51 & 6{,}563{,}840 & 2/2 \\
ConvNeXt-T $b$2 & 28.6 & 241{,}359{,}048 & 238{,}301{,}192 & 1.27 & 0 & 0.89 & 903{,}168 & 2/2 \\
RegNetY-400MF $b$8 & 4.3 & 246{,}797{,}672 & 246{,}797{,}672 & 0.00 & 0 & -0.59 & 1{,}460{,}360 & 2/2 \\
Swin-T $b$2 & 28.3 & 220{,}170{,}192 & 213{,}530{,}552 & 3.02 & 0 & 2.36 & 1{,}449{,}624 & 2/2 \\
\bottomrule
\end{tabular}
}
\end{table}

\subsection{Inference on CPU}
\Cref{tab:appinf} lists the 22 inference configurations of \cref{sec:inference}.
Every reordered program has zero prediction error and bitwise-identical outputs.
DeepLabV3 is the only configuration where the stock order is not optimal.
RegNetY-400MF is the only configuration where the bounds leave a gap: the order found there equals the stock order, and the bound is 13\% below it.

\begin{table}[h]
\centering\scriptsize
\caption{Inference on CPU (fp32, evaluation mode). Peak of the stock and of the reordered program, the best lower bound, whether it certifies the order, and bitwise equality with the stock program.}
\label{tab:appinf}
\begin{tabular}{@{}lrrrrrcc@{}}
\toprule
Configuration & Units & Stock (MB) & Ours (MB) & Saving \% & $\max(\mathcal F,\mathcal R)$ (MB) & Cert. & Bitwise \\
\midrule
BERT-base, $b$8 $s$512 & 133 & 75.50 & 75.50 & 0.0 & 75.50 & yes & yes \\
ConvNeXt-T, $b$8 & 115 & 58.00 & 58.00 & 0.0 & 58.00 & yes & yes \\
DeepLabV3-R50, $b$2, 512$^2$ & 119 & 167.77 & 150.99 & 10.0 & 150.99 & yes & yes \\
DenseNet-121, $b$8 & 242 & 144.51 & 144.51 & 0.0 & 144.51 & yes & yes \\
DistilBERT, $b$8 $s$512 & 67 & 75.50 & 75.50 & 0.0 & 75.50 & yes & yes \\
EfficientNet-B0, $b$8 & 164 & 48.17 & 48.17 & 0.0 & 48.17 & yes & yes \\
TransformerEnc 2L $d$64 & 2 & 0.39 & 0.39 & 0.0 & 0.39 & yes & yes \\
TransformerEnc 4L, $b$4 $s$1024 & 4 & 167.77 & 167.77 & 0.0 & 167.77 & yes & yes \\
Transformer enc--dec $d$256 & 34 & 1.05 & 1.05 & 0.0 & 1.05 & yes & yes \\
FCN-R50, $b$2, 512$^2$ & 107 & 171.97 & 171.97 & 0.0 & 171.97 & yes & yes \\
GoogLeNet, $b$8 & 89 & 38.54 & 38.54 & 0.0 & 38.54 & yes & yes \\
GPT-2 small, $b$4 $s$1024 & 109 & 130.84 & 130.84 & 0.0 & 130.84 & yes & yes \\
Inception-v3, $b$8 & 172 & 66.54 & 66.54 & 0.0 & 66.54 & yes & yes \\
Llama 4L, $b$4 $s$1024 & 57 & 92.29 & 92.29 & 0.0 & 92.29 & yes & yes \\
Llama 8L, $b$2 $s$2048 & 113 & 146.80 & 146.80 & 0.0 & 146.80 & yes & yes \\
LR-ASPP-MNv3, $b$4, 512$^2$ & 132 & 90.83 & 90.83 & 0.0 & 90.83 & yes & yes \\
MobileNetV2, $b$16 & 106 & 96.34 & 96.34 & 0.0 & 96.34 & yes & yes \\
RegNetY-400MF, $b$16 & 168 & 73.87 & 73.87 & 0.0 & 64.24 & no & yes \\
ResNet-50, $b$8 & 104 & 57.80 & 57.80 & 0.0 & 57.80 & yes & yes \\
Swin-T, $b$8 & 154 & 97.76 & 97.76 & 0.0 & 97.76 & yes & yes \\
T5-small enc., $b$8 $s$512 & 63 & 100.96 & 100.96 & 0.0 & 100.96 & yes & yes \\
ViT-B/16, $b$16 & 123 & 97.07 & 97.07 & 0.0 & 97.07 & yes & yes \\
\bottomrule
\end{tabular}

\end{table}

\subsection{GPU replication}
\Cref{tab:appgpu} lists the 50 GPU configurations of \cref{sec:gpu}.
``Det.\ stock'' records whether two stock runs are bitwise identical.
Where they are not, ``within'' means that the reordered program differs from the stock one by no more than stock differs from itself.

{\scriptsize
\setlength{\tabcolsep}{3pt}
\begin{longtable}{@{}lrrrrrcc@{}}
\caption{GPU replication on an RTX 2080 Ti. Peak allocated by the CUDA caching allocator inside the measured step.}\label{tab:appgpu}\\
\toprule
Configuration & Stock (MB) & Ours (MB) & Saving \% & Pred.\ err (B) & Regen.\ \% & Det.\ stock & Bitwise \\
\midrule
\endhead
bert-base-b1-s128 & 438.79 & 438.79 & 0.00 & 0 & -0.45 & no & within \\
bert-base-b2-s256 & 468.94 & 457.02 & 2.54 & 11{,}667{,}456 & 2.10 & no & within \\
convnext-tiny-b2 & 239.08 & 235.99 & 1.29 & 368{,}640 & 0.92 & no & within \\
densenet121-b2 & 261.49 & 261.49 & 0.00 & 0 & -2.59 & no & within \\
distilbert-b1-s128 & 270.12 & 270.12 & 0.00 & 0 & -0.73 & yes & yes \\
distilbert-b2-s256 & 288.08 & 282.84 & 1.82 & 6{,}436{,}352 & 3.00 & yes & yes \\
distilbert-b4-s512 & 658.17 & 658.17 & 0.00 & -- & 0.00 & yes & yes \\
efficientnet-b0-b4 & 226.04 & 226.04 & 0.00 & 0 & -1.36 & no & within \\
enc-L2-d256-s1024 & 33.69 & 32.47 & 3.62 & 917{,}504 & 3.62 & yes & yes \\
enc-L2-d256-s128 & 6.85 & 6.84 & 0.07 & 0 & 0.07 & yes & yes \\
enc-L2-d256-s256 & 8.69 & 8.67 & 0.20 & 0 & 0.20 & yes & yes \\
enc-L2-d256-s512 & 17.12 & 16.31 & 4.71 & 0 & 4.71 & yes & yes \\
enc-L2-d64-s128 & 1.18 & 1.13 & 4.17 & 0 & 4.17 & yes & yes \\
enc-L4-d512-s512 & 66.14 & 65.08 & 1.60 & 1{,}089{,}536 & 1.60 & yes & yes \\
enc-L4-d64-s128 & 2.04 & 1.99 & 2.41 & 0 & 2.41 & yes & yes \\
enc-L6-d256-s256 & 22.64 & 22.33 & 1.34 & 0 & 1.34 & yes & yes \\
encdec-L2-d256 & 15.14 & 14.95 & 1.30 & 0 & 0.00 & no & within \\
encdec-L2-d64 & 1.14 & 1.05 & 7.39 & 0 & 7.39 & yes & yes \\
gpt2-small-b1-s256 & 504.32 & 504.32 & 0.00 & 786{,}432 & 0.00 & yes & yes \\
gpt2-small-b2-s512 & 712.70 & 712.70 & 0.00 & 9{,}435{,}648 & 0.00 & yes & yes \\
llama-L4-h512-b1-s256 & 119.30 & 117.72 & 1.32 & 0 & -0.88 & yes & yes \\
llama-L4-h512-b2-s512 & 154.52 & 152.67 & 1.20 & 2{,}112{,}000 & 1.20 & yes & yes \\
llama-L8-h768-b1-s512 & 344.38 & 346.87 & -0.72 & 20{,}936{,}192 & -1.64 & yes & yes \\
mobilenetv2-b8 & 441.12 & 447.24 & -1.39 & 7{,}771{,}648 & -3.79 & no & within \\
regnet-y400-b8 & 250.36 & 250.36 & 0.00 & 231{,}424 & -0.73 & no & within \\
resnet18-b8 & 253.18 & 252.98 & 0.08 & 524{,}288 & 0.08 & no & within \\
resnet50-b4 & 418.16 & 418.21 & -0.01 & 2{,}801{,}664 & -1.91 & no & within \\
sweep-bert-L4-b1-s128 & 211.33 & 211.33 & 0.00 & 0 & -0.93 & no & within \\
sweep-bert-L4-b1-s256 & 217.63 & 216.05 & 0.73 & 1{,}024 & -1.08 & no & within \\
sweep-bert-L4-b1-s512 & 224.05 & 222.78 & 0.57 & 4{,}739{,}584 & 0.57 & no & within \\
sweep-bert-L4-b2-s128 & 216.70 & 218.15 & -0.67 & 4{,}063{,}232 & -1.03 & no & within \\
sweep-bert-L4-b2-s256 & 227.06 & 225.49 & 0.69 & 5{,}263{,}872 & 1.44 & no & within \\
sweep-bert-L4-b2-s512 & 234.71 & 235.71 & -0.43 & 9{,}335{,}808 & -0.32 & no & within \\
sweep-bert-L4-b4-s128 & 227.06 & 223.13 & 1.73 & 3{,}693{,}056 & 2.48 & no & within \\
sweep-bert-L4-b4-s256 & 234.71 & 236.76 & -0.88 & 6{,}315{,}008 & -0.43 & no & within \\
sweep-bert-L4-b4-s512 & 454.50 & 454.50 & 0.00 & 3{,}037{,}184 & 0.00 & no & within \\
sweep-gpt2-L4-b1-s128 & 274.50 & 274.50 & -0.00 & 786{,}432 & -0.00 & yes & yes \\
sweep-gpt2-L4-b1-s256 & 276.07 & 276.07 & 0.00 & 786{,}432 & 0.00 & yes & yes \\
sweep-gpt2-L4-b1-s512 & 282.89 & 282.89 & 0.00 & 1{,}594{,}368 & 0.00 & yes & yes \\
sweep-gpt2-L4-b2-s128 & 274.50 & 274.50 & 0.00 & 1{,}966{,}080 & 0.00 & yes & yes \\
sweep-gpt2-L4-b2-s256 & 277.28 & 281.45 & -1.51 & 4{,}192{,}256 & -1.79 & yes & yes \\
sweep-gpt2-L4-b2-s512 & 285.96 & 285.96 & 0.00 & 9{,}485{,}312 & 0.00 & yes & yes \\
sweep-gpt2-L4-b4-s128 & 276.88 & 279.66 & -1.00 & 9{,}781{,}248 & -1.24 & yes & yes \\
sweep-gpt2-L4-b4-s256 & 284.65 & 284.03 & 0.22 & 9{,}998{,}336 & -0.20 & yes & yes \\
sweep-gpt2-L4-b4-s512 & 507.99 & 507.99 & 0.00 & 9{,}449{,}984 & 0.00 & no & within \\
swin-t-b2 & 223.62 & 218.94 & 2.10 & 314{,}368 & 0.59 & no & within \\
t5enc-small-b1-s512 & 153.13 & 149.98 & 2.06 & 0 & 1.37 & no & within \\
t5enc-small-b2-s128 & 144.47 & 144.47 & 0.00 & 0 & -6.26 & no & within \\
vit-base-b2 & 355.61 & 355.10 & 0.14 & 10{,}948{,}608 & -0.06 & yes & yes \\
vit-small-b4 & 128.93 & 125.37 & 2.77 & 0 & 2.77 & yes & yes \\
\bottomrule
\end{longtable}

}

\subsection{Microcontroller-class models}
\Cref{tab:apptiny} lists the 16 models of \cref{sec:tiny}.
In every model the order in which TFLite Micro executes the operators is certified optimal, and the peak equals the inputs plus the output of a single operator.

\begin{table}[h]
\centering\scriptsize
\caption{Microcontroller-class models (batch 1, int8 activations). Peak live activations and greedy arena of the model order, whether the order is certified optimal, and whether executing the schedule is bitwise exact.}
\label{tab:apptiny}
\begin{tabular}{@{}lrrrrcc@{}}
\toprule
Model & Ops & MMACs & Peak (KB) & Arena (KB) & Order opt. & Exact \\
\midrule
FC autoencoder (AD) & 11 & 0.3 & 0.8 & 0.8 & yes & yes \\
EfficientNet-B0, 128$^2$ & 134 & 130.4 & 491.5 & 491.5 & yes & yes \\
GoogLeNet, 96$^2$ & 82 & 277.1 & 184.3 & 184.3 & yes & yes \\
ResNet-8 (CIFAR) & 15 & 12.6 & 49.2 & 49.2 & yes & yes \\
DS-CNN (KWS) & 12 & 3.3 & 19.2 & 19.2 & yes & yes \\
TC-ResNet8-1.5 & 16 & 3.4 & 6.5 & 6.5 & yes & yes \\
TC-ResNet8 & 16 & 1.6 & 5.7 & 5.7 & yes & yes \\
MobileNetV2-0.35, 160$^2$ & 65 & 32.4 & 384.0 & 384.0 & yes & yes \\
MobileNetV2-0.35, 96$^2$ & 65 & 12.5 & 138.2 & 138.2 & yes & yes \\
MobileNetV3-S, 128$^2$ & 80 & 20.4 & 114.7 & 114.7 & yes & yes \\
MNASNet-0.5, 128$^2$ & 65 & 36.0 & 131.1 & 163.8 & yes & yes \\
Tiny U-Net, 128$^2$ & 24 & 161.5 & 786.4 & 786.4 & yes & yes \\
Tiny U-Net, 96$^2$ & 24 & 90.9 & 442.4 & 442.4 & yes & yes \\
ShuffleNetV2-0.5, 128$^2$ & 92 & 14.5 & 147.5 & 147.5 & yes & yes \\
SqueezeNet1.1, 128$^2$ & 39 & 106.5 & 315.5 & 315.5 & yes & yes \\
MobileNetV1-0.25 (VWW) & 30 & 7.7 & 55.3 & 64.5 & yes & yes \\
\bottomrule
\end{tabular}

\end{table}

\subsection{Speaker-embedding networks and detectors}
\Cref{tab:apprdn} completes \cref{tab:rdnpub} with the 0.8\,s inputs and the greedy arena.
For B3 and B6 the greedy planner does not realize the saving, because of fragmentation.
\Cref{tab:appyolo} reports the YOLO detectors at $640\times640$.
In YOLOv5nu and YOLOv8n the peak is the input image plus the output of the first convolution.
In the other models it is a concatenation in the feature pyramid, whose large inputs are both needed at the peak event.
Selective recomputation therefore has almost nothing to work with, which illustrates the condition stated in \cref{sec:rdn}.

\begin{table}[h]
\centering\scriptsize
\caption{Public ReDimNet backbones: peak of the certified optimal order, after recomputing the stem (and the saving), greedy arena before and after, and extra MACs.}
\label{tab:apprdn}
\begin{tabular}{@{}lrrrrrr@{}}
\toprule
Size & Frames & Opt.\ peak & Recompute & Arena before & Arena after & MACs \% \\
\midrule
B0 & 201 & 7.28 & 6.85 ($-$6.0\%) & 7.29 & 6.85 & 0.33 \\
B0 & 81 & 2.96 & 2.79 ($-$6.0\%) & 2.97 & 2.79 & 0.34 \\
B1 & 201 & 10.49 & 9.85 ($-$6.1\%) & 10.49 & 9.85 & 0.33 \\
B1 & 81 & 4.27 & 4.01 ($-$6.1\%) & 4.27 & 4.01 & 0.34 \\
B2 & 201 & 16.78 & 15.91 ($-$5.2\%) & 16.79 & 15.91 & 0.25 \\
B2 & 81 & 6.83 & 6.47 ($-$5.2\%) & 6.83 & 6.48 & 0.25 \\
B3 & 201 & 16.78 & 15.91 ($-$5.2\%) & 16.79 & 16.85 & 0.09 \\
B3 & 81 & 6.83 & 6.47 ($-$5.2\%) & 6.83 & 6.86 & 0.09 \\
B4 & 201 & 33.56 & 31.76 ($-$5.4\%) & 33.57 & 31.77 & 0.11 \\
B4 & 81 & 13.66 & 12.93 ($-$5.4\%) & 13.67 & 12.94 & 0.11 \\
B5 & 201 & 33.56 & 31.76 ($-$5.4\%) & 33.57 & 31.77 & 0.06 \\
B5 & 81 & 13.66 & 12.93 ($-$5.4\%) & 13.67 & 12.94 & 0.06 \\
B6 & 201 & 33.56 & 31.76 ($-$5.4\%) & 33.57 & 33.63 & 0.02 \\
B6 & 81 & 13.66 & 12.93 ($-$5.4\%) & 13.67 & 13.69 & 0.02 \\
\bottomrule
\end{tabular}

\end{table}

\begin{table}[h]
\centering\scriptsize
\caption{YOLO detectors at $640\times640$ (fp32, raw head outputs). The model order is certified optimal in every case.}
\label{tab:appyolo}
\begin{tabular}{@{}lrrrcr@{}}
\toprule
Model & Ops & GMACs & Optimal peak (MB) & Cert. & With recomputation (MB) \\
\midrule
YOLOv5nu & 69 & 3.9 & 11.47 & yes & 11.47 ($-$0.0\%) \\
YOLOv8n & 65 & 4.4 & 11.47 & yes & 11.47 ($-$0.0\%) \\
YOLO11n & 101 & 3.3 & 14.34 & yes & 14.34 ($-$0.0\%) \\
YOLOv8s & 68 & 14.3 & 22.12 & yes & 21.71 ($-$1.9\%) \\
YOLO11s & 103 & 10.8 & 28.67 & yes & 28.67 ($-$0.0\%) \\
YOLOv8m & 84 & 39.5 & 39.32 & yes & 39.32 ($-$0.0\%) \\
YOLO11m & 123 & 34.1 & 56.52 & yes & 56.12 ($-$0.7\%) \\
\bottomrule
\end{tabular}

\end{table}

\subsection{Negative rule families for gradient scheduling}
\label{app:negative}
A natural question is whether the savings of \cref{sec:train} follow a simple rule.
Two candidate rules are ``finish a stage's parameter-gradient branch before advancing the backbone'' and ``always defer parameter gradients''.
Both can nearly double the peak.
Let a source produce a private activation $a$ of size $A$ and a tensor $u$ of size 1.
Operation $G$ consumes $a$ and $u$ and creates a pinned gradient of size $P$.
Operation $X$ consumes $u$ and creates $z$ of size 1 with scratch $Q$.
Counting outputs allocated before inputs are released, the two orders have peaks
\[
\mathrm{peak}(GX)=\max(A+1+P,\;P+Q+2),\qquad \mathrm{peak}(XG)=\max(A+Q+2,\;A+P+2).
\]
With $A=1$ and $P=Q=M$, finishing $G$ first costs $2M+2$, while deferring it costs $M+3$.
With $A=Q=M$ and $P=1$ the outcomes are exchanged.
Hence neither rule is safe; each loses a factor approaching two on one of the families.

\subsection{Theory experiments}
\label{app:synthetic}
\label{app:t1}
\Cref{tab:appt1} breaks the comparison of the lower bounds with exact optima down by family.
The reverse bound is exact on every in-tree and training-shaped instance, which matches \cref{prop:intree}.
\Cref{tab:appt3} reports the number of states that our prototype of the bounded-treewidth algorithm generates.
The growth with $K$ and $B$ is the parameter dependence of \cref{thm:fpt}, and it confirms that the algorithm is a tractability statement rather than a practical tool.

\begin{table}[h]
\centering\scriptsize
\caption{Number of random DAGs on which each lower bound equals the exact optimum, and the worst ratio of the optimum to $\max(\mathcal F,\mathcal R)$.}
\label{tab:appt1}
\begin{tabular}{@{}lrrrrrrr@{}}
\toprule
Family & $n$ & Mandatory & $G$ & $\mathcal F$ & $\mathcal R$ & $\max(\mathcal F,\mathcal R)$ & Worst $\OPT/\max$ \\
\midrule
random & 160 & 63 & 111 & 116 & 143 & 148 & 1.188 \\
sp & 160 & 137 & 145 & 146 & 156 & 157 & 1.231 \\
pumpkin & 160 & 75 & 113 & 116 & 145 & 148 & 1.231 \\
intree & 160 & 142 & 142 & 142 & 160 & 160 & 1.000 \\
training & 160 & 160 & 160 & 160 & 160 & 160 & 1.000 \\
\bottomrule
\end{tabular}

\end{table}

\begin{table}[h]
\centering\scriptsize
\caption{States generated by the complete-trace dynamic program (limit $3\cdot10^6$) on ladder graphs (width 2) and two-arm pumpkins, with $K$ the cap and $B$ the memory budget.}
\label{tab:appt3}
\begin{tabular}{@{}lllrrrrr@{}}
\toprule
Graph & $K$ & $B$ & $n{=}6$ & $n{=}8$ & $n{=}10$ & $n{=}12$ & $n{=}14$ \\
\midrule
ladder & 1 & 4 & 4 & 34 & 17 & 232 & 118 \\
ladder & 1 & 6 & 313 & 268 & 121 & 4{,}849 & 595 \\
ladder & 1 & 8 & 1{,}615 & 3{,}280 & 41{,}656 & 43{,}055 & 2{,}705 \\
ladder & 2 & 4 & 10 & 659 & 169 & 45{,}831 & 17{,}554 \\
ladder & 2 & 6 & 105{,}131 & 128{,}836 & 31{,}570 & $>3{\cdot}10^6$ & 1{,}101{,}844 \\
ladder & 2 & 8 & $>3{\cdot}10^6$ & $>3{\cdot}10^6$ & $>3{\cdot}10^6$ & $>3{\cdot}10^6$ & $>3{\cdot}10^6$ \\
pumpkin & 1 & 4 & 64 & 55 & 55 & 55 & 55 \\
pumpkin & 1 & 6 & 1{,}089 & 2{,}317 & 23{,}100 & 33{,}279 & 40{,}650 \\
pumpkin & 1 & 8 & 6{,}011 & 102{,}948 & 764{,}110 & $>3{\cdot}10^6$ & $>3{\cdot}10^6$ \\
pumpkin & 2 & 4 & 8{,}333 & 2{,}967 & 2{,}967 & 2{,}967 & 2{,}967 \\
pumpkin & 2 & 6 & $>3{\cdot}10^6$ & $>3{\cdot}10^6$ & $>3{\cdot}10^6$ & $>3{\cdot}10^6$ & $>3{\cdot}10^6$ \\
pumpkin & 2 & 8 & $>3{\cdot}10^6$ & $>3{\cdot}10^6$ & $>3{\cdot}10^6$ & $>3{\cdot}10^6$ & $>3{\cdot}10^6$ \\
\bottomrule
\end{tabular}

\end{table}

\label{app:t3}

\section{Exact realization}
\label{app:lift}

\subsection{The contract}
The contract of \cref{thm:lift} describes what a generated program must guarantee for a permutation of its statements to be both correct and exactly predictable.
Its state $Q=(E,G,O,U,H)$ consists of five parts:
\begin{itemize}
\item the live storage epochs $E$, where an epoch is one allocation-to-free interval of a storage;
\item the defined content generations $G$ of byte regions inside those epochs;
\item the owners $O$ of each epoch, namely wrapper names, tuple slots, saved values, native references and returned values;
\item the outstanding asynchronous uses $U$;
\item any externally visible state $H$, such as random-number generators.
\end{itemize}
The six conditions are the following.
\begin{enumerate}[label=(C\arabic*)]
\item Every unit's reads and writes are described at the level of byte regions and content generations, including partial writes through aliases.
\item Every owner is recorded, and every access has completed before the storage it touches is released.
\item Every unit computes a deterministic function of the contents it reads.
\item The sequence of charged allocation and release events of every unit is the recorded one.
\item This sequence is independent of the surrounding live memory: executing a unit next to any disjoint live frame $R$ adds exactly $|R|$ to every event and leaves $R$ untouched.
\item Emitting a new order introduces no events other than those of the units, their allocations and their releases.
\end{enumerate}
An order is \emph{admissible} if it respects every conflicting pair of region accesses and every recorded ownership constraint.

\subsection{Proof of \cref{thm:lift}}
The proof has two parts, one about values and one about memory.

For values, consider two adjacent units in the reference order that have no conflicting accesses.
Swapping them changes neither the contents either of them reads nor the contents left behind, because their written regions are disjoint from each other's read and written regions.
An admissible order can be reached from the reference order by such swaps.
By (C1) and (C3), every unit therefore reads, in the new order, exactly the content generations it read in the reference order and computes the same result.
If the units are deterministic at the byte level the results are bitwise equal.

For memory, we maintain a correspondence between the abstract state of the root model and the concrete state of the program, event by event.
Allocation creates a fresh epoch on both sides.
Binding a name or saving a value adds the same owner on both sides.
Releasing a name removes it on both sides.
A storage is freed only when its owner set is empty and its accesses are complete, which happens at the same event on both sides by (C2).
By (C4) the charged sizes match, and by (C5) the bytes held by units that are not currently executing contribute the same constant to every event.
By (C6) no other events occur.
Hence the live bytes agree after every event, and so do their maxima.
\qed

\subsection{What breaks without each condition}
Each condition is needed; the following small programs violate one condition each.
\begin{itemize}
\item \textbf{Owner-sensitive call (C2).} A root of size 10 has two owners, $a$ and $z$. Unit $D$ drops $z$. Unit $K$ reads the root, drops $a$ and allocates an output of size 6. The order $D,K$ peaks at 10, while $K,D$ peaks at 16, with identical code and inputs.
\item \textbf{Hidden runtime state (C3, C5).} A unit uses scratch 4 or 64 depending on a flag that another unit sets.
\item \textbf{Partial write through an alias (C1).} One unit reads element zero through one view, another overwrites only element zero through another view. Swapping them changes the returned value from 2 to 9.
\item \textbf{Asynchronous completion (C2).} An asynchronous read observes a later overwrite unless completion is ordered before the write.
\item \textbf{Regrouping (C6).} Two units that each allocate and free 8 bytes peak at 8 when executed atomically and at 16 when their events are interleaved.
\item \textbf{Envelope versus workspace.} A unit that frees an input of 10 and then allocates an output of 6 has net change $-4$ and maximal prefix change 0, so the commonly used difference between the two is 4 although the unit has no workspace.
\item \textbf{A late alias} (observed during development). Emitting an alias only at the return statement kept a tuple, and with it a gradient, alive and raised the ResNet-18 backward peak by 16\,KiB. Emitting aliases as early as their dependencies allow restores exactness.
\item \textbf{The caching allocator (C5).} On the GPU the block assigned to a request depends on the history of the allocator's pools (\cref{sec:gpu}), which is the failure we measure there.
\end{itemize}

\section{Trap bounds}
\label{app:traps}

Throughout, $L(X)$ is the live memory after the ideal $X$, $c(v)$ the event cost of $v$ beyond the live memory, $A(v)$ the strict ancestors and $D(v)$ the descendants of $v$ including $v$.

\subsection{Soundness and domination}
\paragraph{Soundness of $\mathcal F$.}
Fix a proper ideal $R$ and an arbitrary topological order.
Let $v$ be the first operation of the order outside $R$.
All ancestors of $v$ precede it and lie in $R$, so $v$ is in the ready frontier $F(R)$.
The set $X$ of operations before $v$ satisfies $A(v)\subseteq X\subseteq R$, so the event of $v$ costs $L(X)+c(v)\ge f_v(R)\ge\min_{u\in F(R)}f_u(R)$.
Taking the maximum over $R$ gives $\mathcal F\le\OPT$.

\paragraph{Soundness of $\mathcal R$.}
Fix a nonempty ideal $T$ and let $v$ be the last operation of $T$ in the order.
Then $v$ is maximal in $T$, and the set of operations before $v$ contains $T\setminus\{v\}$ and contains no descendant of $v$.
The event of $v$ therefore costs at least $q_v(T)$.

\paragraph{Domination of $G$.}
For every $v$ the set $R_v=V\setminus D(v)$ is an ideal whose ready frontier is exactly $\{v\}$.
The ideals $X$ with $A(v)\subseteq X\subseteq R_v$ are exactly the ideals that contain $A(v)$ and exclude $v$.
Hence $f_v(R_v)$ equals the event-cut term $g_v$, and $\mathcal F\ge\max_v g_v=G$.
Symmetrically, $T=A(v)\cup\{v\}$ has the unique maximal element $v$ and the interval $[A(v),V\setminus D(v)]$, so $\mathcal R\ge G$.

\subsection{Computation}
At a capacity $\lambda$, call a ready operation $v$ admissible for $R$ if $f_v(R)\le\lambda$.
Enlarging $R$ can only enlarge the interval over which $f_v$ minimizes, so admissibility persists.
Saturating admissions from $R=\emptyset$ therefore reaches a unique fixed point.
If the fixed point is a proper ideal $R$, every operation of its frontier has $f_v(R)>\lambda$, which certifies $\mathcal F>\lambda$.
Conversely, if some proper ideal $T$ has $\min_{v\in F(T)}f_v(T)>\lambda$, no admission sequence can leave $T$.
The first operation to leave it would be admissible for an ideal contained in $T$, and hence for $T$ itself, a contradiction.
Saturation thus completes if and only if $\lambda\ge\mathcal F$.
The adaptive algorithm avoids a search over $\lambda$: it admits an operation of minimum $f_v(R)$ at every step and returns the largest value it recorded.
It performs $n$ admissions with at most $n$ frontier evaluations each.
Each evaluation is a minimum-weight closure: operations carry their allocations, a release node per root carries a gain equal to the root's size, and infinite arcs encode precedence and the interval endpoints.
This gives $O(n^2)$ minimum-cut computations.
The reverse bound removes maximal operations from $V$ in the same way; its converse argument uses the first removed member of any nonempty ideal.

\subsection{Incomparability}
\begin{itemize}
\item \textbf{$\mathcal R>\mathcal F$.} On the backward wrapper of DistilBERT ($b1$, $s128$) the forward bound is 266{,}237{,}960\,B, while the reverse bound and the optimum are both 266{,}628{,}104\,B.
\item \textbf{$\mathcal F>\mathcal R$.} Let two independent operations each consume their own external input of size 2 and use scratch 1, with initial memory 4. Whichever runs first costs 5, so $\mathcal F=\OPT=5$. In the reverse bound each operation may assume that the other has already run, so $\mathcal R=3$, and 4 with the boundary terms.
\end{itemize}

\subsection{Two families}

\begin{proposition}\label{prop:intree}
Let $\mathcal T_h$ be a complete binary in-tree of height $h\ge1$ with unit weights, a zero-weight source feeding every leaf, and a pinned root.
Then $\mathcal F=3$ and $\mathcal R=\OPT=h+2$.
\end{proposition}
\begin{proof}
\emph{$\OPT=h+2$.} Evaluating one child subtree, keeping its root, then evaluating the other and finally the parent gives $h+2$ by induction.
For the lower bound, restrict any order to each child subtree; by induction each restriction contains an event with at least $h+1$ live values.
At the later of these two events the other subtree has already started and its parent has not yet run, so the other subtree holds at least one live value, which gives $h+2$.

\emph{$\mathcal F=3$.} Every internal vertex is admissible at capacity 3 using the ideal of its own ancestors, in which exactly its two children are live.
Hence saturation completes at capacity 3.
The lowest internal vertex shows $\mathcal F\ge G\ge 3$.

\emph{$\mathcal R\ge h+2$.} Let $T$ consist of the source and all vertices of height at most one; its maximal elements are the vertices of height one.
Fix such a vertex $v$ and an ideal $X$ with $T\setminus\{v\}\subseteq X\subseteq V\setminus D(v)$; here $D(v)$ consists of $v$ and its $h-1$ strict ancestors.
Both children of $v$ lie in $X$ and are live because $v\notin X$.
For each strict ancestor $a$ of $v$, let $a'$ be its child off the path to $v$.
The subtree of $a'$ meets $X$, because it contains vertices of height at most one.
The restriction of $X$ to that subtree is a nonempty ideal of an in-tree whose root's consumer $a$ is not in $X$, so some maximal element of the restriction has an unexecuted consumer and is live.
These $h-1$ subtrees are disjoint from each other and from the children of $v$.
Hence $L(X)\ge2+(h-1)$ and $q_v(T)\ge h+2$, and soundness gives $\mathcal R\le\OPT$.
\end{proof}

\begin{proposition}\label{prop:merge}
Let a zero-weight source feed $k\ge2$ unit-weight values $a_1,\dots,a_k$.
Let a zero-weight merge $m_{ij}$ consume $a_i$ and $a_j$ for every pair $i<j$, and let a zero-weight sink consume all merges.
Then $\mathcal F=2$ and $\mathcal R=\OPT=k$.
\end{proposition}
\begin{proof}
Computing all $a_i$ first and then all merges has peak $k$, so $\OPT\le k$.
For $\mathcal F\le2$, consider a proper ideal $R$.
If some $a_i\notin R$, it is ready with prefix $\{s\}$ and costs 1.
If all $a_i\in R$ but some merge $m_{ij}$ is not, its prefix $\{s,a_i,a_j\}$ has two live values.
Otherwise only the sink is outside $R$ and costs nothing.
Every merge shows $G\ge2$.
For $\mathcal R\ge k$, let $T=\{s,a_1,\dots,a_k\}$, whose maximal elements are the $a_i$, and let $X$ satisfy $T\setminus\{a_i\}\subseteq X\subseteq V\setminus D(a_i)$.
Every $a_l$ with $l\ne i$ lies in $X$, and its consumer $m_{il}$ lies in $D(a_i)$ and hence outside $X$, so $a_l$ is live.
Therefore $q_{a_i}(T)\ge1+(k-1)=k$.
\end{proof}

Exhaustive computation confirms $\mathcal R=\OPT$ also with merges of weight one for $k\le6$, where $\mathcal F=\binom k2$ and $\mathcal R=\OPT=\binom k2+2$.
Whether $\OPT/\max(\mathcal F,\mathcal R)$ is bounded on all DAGs remains open; the largest ratio we found is 1.5.

\section{Necessary recomputation}
\label{app:remat}

\begin{lemma}[Projection]\label{lem:proj}
Let $P\subseteq V$ and let $Z_P$ be the one-shot instance with the weights of all vertices outside $P$ set to zero and the dependencies unchanged.
If a capped schedule computes every vertex of $P$ exactly once with peak at most $B$, then $\OPT_1(Z_P)\le B$.
\end{lemma}
\begin{proof}
The order of first computations is a topological order of $Z_P$.
A value of $P$ that is computed only once cannot be deleted while it is still needed, since it could not be recreated.
It is therefore resident at every event of $Z_P$ at which the one-shot schedule keeps it, and each one-shot event costs no more than the corresponding capped event.
\end{proof}

\paragraph{The cover bound.}
For an event $h$ let $\ell(h)=w_h+\sum_{p\in\mathrm{pred}(h)}w_p$ be its local footprint.
Let $X_h$ be the values produced by strict ancestors of $h$ that have a consumer among the strict descendants of $h$, excluding the inputs of $h$.
Every value in $X_h$ that is computed only once must be resident at the first execution of $h$.
With $z_v=\mathbf 1[N_v\ge2]$ this gives the knapsack-cover constraint
\[
\textstyle\sum_{v\in X_h}w_v z_v\;\ge\;\ell(h)+\sum_{v\in X_h}w_v-B.
\]
Its fractional minimum of $\sum_v c_v z_v$, obtained greedily by increasing $c_v/w_v$, is a lower bound on the extra work.
The maximum over all events is the cover bound evaluated in \cref{sec:remat}.
The constraints of all events can also be combined in one covering linear program; a rational dual solution then certifies the bound without trusting the solver.

\section{Hardness}
\label{app:hard}

\subsection{Proof of \cref{thm:hard}}
\paragraph{Membership.} With a uniform cap $k$, a normalized schedule has at most $kn$ computations and at most $kn$ deletions, so it is a polynomial certificate whose peak can be checked in polynomial time.

\paragraph{Hardness.} We reduce from weighted one-shot pebbling on pumpkin graphs, which is strongly NP-hard~\cite{jin2023newtools}.
Let $(G,B)$ be an instance with sink $t$.
First, make all weights positive: replace every weight $w_v$ by $(n+1)w_v+1$ and $B$ by $(n+1)B+n$.
At most $n$ values coexist, so this preserves the answer, and the numbers remain polynomial in unary.
Next, append the following gadget:
\begin{itemize}
\item zero-weight vertices $r_1,\dots,r_k$ and $z_1,\dots,z_{k-1}$, and barrier vertices $h_1,\dots,h_{k-1}$ of weight $B$;
\item edges $t\to r_i$ for every $i$, and $r_i\to h_i\to z_i\to r_{i+1}$ for $i<k$;
\item the new sink is $r_k$, and every vertex receives the cap $k$.
\end{itemize}

\emph{If.} Run a one-shot schedule of $G$ with peak at most $B$, compute $r_1$, delete every old value, compute $h_1$ and $z_1$, delete $h_1$, and repeat.
Each phase stays within $B$, and every vertex is computed at most $k$ times.

\emph{Only if.} At the first computation of each barrier $h_i$, no old value can be resident, because all old weights are positive and $h_i$ alone has weight $B$.
The first computation of $r_{i+1}$ requires a fresh $t$, and hence a computation of every old vertex, after that barrier, and the first $r_1$ requires one before the first barrier.
These $k$ phases are disjoint, so with cap $k$ every old vertex is computed exactly once in each phase.
Restricting any one phase to the old events yields a one-shot schedule of $G$ with peak at most $B$.

\emph{Treewidth.} A pumpkin has treewidth two.
Attach to a bag containing $t$ the chain of bags $\{t,r_i,h_i\}$, $\{t,h_i,z_i\}$, $\{t,z_i,r_{i+1}\}$ for $i=1,\dots,k-1$.
These bags cover all new edges, every vertex occurs in a connected set of bags, and no bag has more than three vertices.
\qed

\subsection{Binary caps}
If two moments of an execution have the same live set, the actions between them can be removed: the remaining schedule is still valid, its counters only decrease and its peak does not increase.
Hence caps above $2^n$ are never needed, and with caps written in binary the problem embeds unrestricted black pebbling, which is PSPACE-complete~\cite{gilbert1980pebbling,chan2023pebbling}.

\subsection{Series--parallel graphs}
\paragraph{The fork/join region.}
Given items $i$ with weights $w_i$ and prices $c_i$, build a binary tree of zero-weight forks $f_I$ and joins $j_I$ of weight $W=C+1$, whose leaves are the item vertices $a_i$ of weight $w_i$ and price $c_i$.
The region is two-terminal series--parallel with out-degree two.
All of its values fit in $B=\sum_i w_i+(n-1)W$.

\paragraph{Knapsack equality for minimum work with uniform cap two.}
Append $t\to r_1\to H\to z\to r_2$ and $t\to r_2$, where $H$ has weight $B-C$ and the other new vertices have weight zero.
At the first computation of $H$ no join fits beside it, so the resident item vertices form a set $S$ with $\sum_{i\in S}w_i\le C$.
After $H$, every join and every item outside $S$ must be computed again to produce $t$ for $r_2$.
Conversely, every set $S$ with $\sum_S w_i\le C$ can be realized this way.
Hence the minimum extra work equals $\sum_i c_i-\max\{\sum_{i\in S}c_i:\sum_{i\in S}w_i\le C\}$.

\paragraph{Partition with caps at most three.}
Require three results of the region, separated by two barriers of weight $B-C$.
Give item vertices cap two and all other vertices cap three.
An item that is absent at both barriers would need three computations, so the item sets $S_1,S_2$ resident at the two barriers cover all items.
Each has weight at most $C$, and the total weight is $2C$, so they form an exact partition.

\paragraph{An exponential trade-off frontier.}
Take the path $s\to u_1\to\dots\to u_m\to g\to h\to q\to v_m\to\dots\to v_1$ with the additional edges $u_i\to v_i$.
Let $w(u_i)=c(u_i)=a_i=2^{i-1}$, $w(h)=H\ge A=\sum_i a_i$, and give all other vertices weight and price zero.
Keeping a subset $S$ of the $u_i$ resident through $h$ and recomputing the others realizes the pair $(M,E)=(H+a(S),\,A-a(S))$ of peak and extra work.
Every schedule satisfies $M+E\ge H+A$, so all $2^m$ pairs are non-dominated.

\section{The bounded-treewidth algorithm}
\label{app:fpt}

\subsection{Why the obvious interface fails}
Take vertices $s,u,a,h,c,t$ with weights $0,10,1,100,1,1$, edges $s\to u\to a\to h\to c\to t$ and $u\to t$, and cap two everywhere.
In history A every vertex is computed once and $u$ is kept through $h$, with peak 111.
In history B $u$ is discarded before $h$ and recomputed afterwards, with peak 101.
Both end with the same resident boundary values and the same boundary counters, yet history B has used up both computations of $u$.
A later reset of weight 111 that requires the region again can be completed after history A but not after history B.
Hence resident boundary values and boundary counters do not form a sufficient interface, even when no internal value is resident.

\subsection{Local witnesses and types}
Fix a nice tree decomposition of width $t$, let $h=t+1$, and let $X_x$, $V_x$ and $I_x=V_x\setminus X_x$ be the bag, the vertices below node $x$ and its private vertices.
A \emph{local witness} at $x$ is a complete event sequence over $V_x$ that starts and ends empty.
It computes every private vertex between one and its cap times, respects every edge inside $V_x$, and stays within $B$.
Its \emph{boundary word} is the subsequence of compute and delete events of bag vertices, of length at most $2Kh$.
Between consecutive boundary events, and before the first and after the last, we record the sequence of private resident weights and reduce it to its typical sequence~\cite{bodlaender2020typical}, which has length at most $2B+1$.
The \emph{type} of a witness is its boundary word together with these reduced profiles; the number of types is at most $\mathcal W(h,K)\,\mathcal T(B)^{2Kh+1}$ with $\mathcal W\le\sum_{\ell\le2Kh}(2h)^\ell$ and $\mathcal T(B)=(B+1)^{2B+1}$.
For each type the algorithm retains an actual witness of minimum private work.

\subsection{The invariant}
The algorithm maintains \emph{word-preserving segmentwise one-way coverage}: for every admissible local witness there is a retained witness with the same boundary word, no more private work, and a private profile that stutter-dominates the other one in every segment.
Two facts make this invariant compose.
First, typical sequences are compatible with concatenation and with constant shifts, and they admit a merge lemma for sums of independent profiles~\cite{bodlaender2020typical}.
Second, private vertices interact with the rest of the schedule only through memory, and their caps have already been checked on the retained witness.
Replacing a witness by a retained one with the same word therefore keeps every surrounding schedule valid without increasing its peak or its work.
Typical equivalence is used only to deduplicate; the induction relies on one-way domination.

\subsection{Transitions}
\begin{itemize}
\item \textbf{Leaf.} The empty witness.
\item \textbf{Introduce $v$.} No private vertex is adjacent to $v$, so only edges inside the bag need checking. For every retained child witness, a dynamic program over the position in that literal witness and the type of the prefix built so far inserts up to $k_v$ compute/delete pairs of $v$. Equal prefix types at the same position have the same future, so one representative suffices; this needs $O(Kn)$ positions times a parameter-bounded number of types, not $n^{K}$ placements.
\item \textbf{Forget $v$.} Relabel the events of $v$ as private, which concatenates segments and shifts their levels by $w_v$ where $v$ is resident, require that $v$ was computed at least once, and charge $c_v$ times its number of computations.
\item \textbf{Join.} Combine two retained witnesses with identical boundary words: synchronize the bag events, interleave the private events of the two children by a dynamic program over pairs of positions and the merged prefix type, count resident bag values once, and add the private work.
\item \textbf{Root.} At the empty root every vertex is private, so the retained witness is a complete schedule; truncating it after the first computation of the sink gives the answer.
\end{itemize}
Every step uses polynomially many literal positions times a number of types that depends only on $t$, $K$ and $B$, which gives the bound $f(t,K,B)\,\poly(n,L)$ of \cref{thm:fpt}.

\subsection{Comparison-restricted augmentation}
Let $\OPT(B,p)$ be the least work over schedules that never have more than $p$ resident values of positive weight at an allocation event.
For $0<\varepsilon\le1$ round every positive weight to $\lceil w_v/\delta\rceil$ with $\delta=\varepsilon B/p$, and solve exactly with budget $B'=\lfloor p/\varepsilon\rfloor+p$.
Every comparison schedule fits in $B'$, because rounding adds at most $p$ units at each event.
The returned schedule respects all caps, its true peak is at most $\delta B'\le(1+\varepsilon)B$, and its work is at most $\OPT(B,p)$.
The running time is $f(t,K,p,1/\varepsilon)\,\poly(n,L)$.

\section{Supplementary material}
\label{app:artifact}
The supplementary archive reproduces every number, table and figure that concerns public models, and it contains all result files behind them.
\begin{itemize}
\item \texttt{experiments/harness/}: fixed-wrapper realization for Inductor: wrapper capture, root models, emission, regeneration and the per-configuration driver.
\item \texttt{experiments/edge/}: operator-level graphs, schedule execution, the arena planner and selective recomputation.
\item \texttt{experiments/*.py}: the study drivers, the CP-SAT baseline, the allocator simulation, and the scripts that regenerate all tables and figures from the saved results.
\item \texttt{experiments/runs*/}, \texttt{edge/results/}, \texttt{edge/rdnpub/}, \texttt{theory/}: the saved results.
\item \texttt{solvers/}: the exact capacity search, the forward and reverse trap bounds, their independent verifiers, the CP-SAT model with recomputation caps, and the prototype of the bounded-treewidth algorithm.
\end{itemize}
The \texttt{README.md} maps every section, table and figure to the command that produces it.
The certificate verifiers and the table scripts use only the Python standard library.

\end{document}